\documentclass[conference]{IEEEtran}
\IEEEoverridecommandlockouts
\usepackage{amsmath,amsfonts}
\usepackage{array}
\usepackage{textcomp}
\usepackage{stfloats}
\usepackage{url}
\usepackage{amsmath, amssymb, amsthm}
\usepackage{verbatim}
\usepackage{graphicx}
\usepackage{amssymb}
\usepackage{hyperref}
\usepackage{graphicx}
\usepackage{amsmath}
\usepackage{longtable}
\usepackage{algorithm} 
\usepackage{algpseudocode} 
\usepackage{mathrsfs}
\usepackage{subcaption}
\usepackage{mathtools}
\usepackage{pifont}
\usepackage{color}
\usepackage{lineno}
\usepackage{graphicx}  
\usepackage{makecell} 
\usepackage{pdflscape}
\usepackage{adjustbox}
\usepackage[utf8]{inputenc}
\usepackage{tabularx}
\usepackage{blindtext}
\usepackage{longtable}
\usepackage{lscape}
\usepackage{amsthm}
\usepackage{graphicx}
\usepackage{booktabs}
\usepackage{amsmath, amssymb}
\usepackage{amsthm}
\usepackage[numbers]{natbib}
\usepackage{tikz}
\usepackage{siunitx}
\usepackage{mathrsfs}
\usepackage{multirow}
\usetikzlibrary{shapes,arrows}
\usepackage{xcolor}
\usepackage{setspace}
\usepackage{notoccite} 
\usepackage{lscape} 
\usepackage{mwe}
\usepackage{booktabs}
\usepackage{amsthm}
\newtheorem{theorem}{Theorem}[section]

\theoremstyle{definition}

\theoremstyle{definition}

\newcommand{\RNum}[1]{\lowercase\expandafter{\romannumeral #1\relax}}
\newcommand{\RNumU}[1]{\uppercase\expandafter{\romannumeral #1\relax}}
\usepackage[numbers]{natbib}

\def\BibTeX{{\rm B\kern-.05em{\sc i\kern-.025em b}\kern-.08em
    T\kern-.1667em\lower.7ex\hbox{E}\kern-.125emX}}
\begin{document}
\title{ECA-BLS: An Efficient Complex-Augmented Broad Learning System}

\author{\IEEEauthorblockN{A. Rahaman}
\IEEEauthorblockA{\textit{Department of Mathematics} \\
\textit{Indian Institute of Technology Indore}\\
Indore, India \\
phd2401141001@iiti.ac.in}
\and
\IEEEauthorblockN{A. Quadir}
\IEEEauthorblockA{\textit{Department of Mathematics} \\
\textit{Indian Institute of Technology Indore}\\
Indore, India \\
mscphd2207141002@iiti.ac.in}
\and
\IEEEauthorblockN{M. Sajid}
\IEEEauthorblockA{\textit{Department of Mathematics} \\
\textit{Indian Institute of Technology Indore}\\
Indore, India \\
phd2101241003@iiti.ac.in}
\and
\IEEEauthorblockN{M. Akhtar}
\IEEEauthorblockA{\textit{Department of Mathematics} \\
\textit{Indian Institute of Technology Indore}\\
Indore, India \\
phd2101241004@iiti.ac.in}
\and
\IEEEauthorblockN{M. Tanveer}
\IEEEauthorblockA{\textit{Department of Mathematics} \\
\textit{Indian Institute of Technology Indore}\\
Indore, India \\
mtanveer@iiti.ac.in}
}
\maketitle
\begin{abstract}
Broad Learning System (BLS) is an efficient alternative to deep architectures due to its fast training, analytical learning, and strong generalization under limited data. However, existing BLS variants are confined to real-valued representations, restricting their ability to capture nonlinear interactions and second-order statistical dependencies inherent in real-world data. Notably, no prior BLS model fully exploits the complete second-order statistics that naturally emerge when data are embedded in the complex domain. To address this limitation, this paper introduces the first complex augmented  Broad Learning System (CA-BLS), which transforms real-valued inputs into phase-encoded complex representations and adopts widely linear modeling to jointly leverage covariance and pseudo-covariance information via complex conjugate augmentation. This enables effective modeling of latent nonlinearities, coherence structures, and second-order dependencies inaccessible to conventional BLS formulations. To mitigate the additional computational cost of complex augmentation, an Efficient Complex Augmented  BLS (ECA-BLS) is further developed, reformulating CA-BLS entirely in the real domain while preserving its exact decision function, achieving up to 75\% fewer multiplications and over 60\% fewer additions. A rigorous theoretical analysis proves the mathematical equivalence between CA-BLS and ECA-BLS, ensuring zero theoretical loss. Extensive experiments on 26 benchmark datasets from the UCI and KEEL repositories demonstrate that ECA-BLS consistently outperforms classical BLS and recent state-of-the-art randomized neural networks in accuracy, average rank, and statistical significance, establishing augmented second-order modeling as a critical and previously missing dimension of BLS research.
\end{abstract}
\section{Introduction and Motivation}

\IEEEPARstart{I}n today's world, deep learning has become an inevitable research focus in the field of artificial intelligence to handle complex tasks such as image classification \cite{sepehri2024hierarchical, quadir2026hypergraph} and natural language processing \cite{10413606}, and so on. Although deep learning has powerful learning capabilities, its complicated architecture and a vast number of hyperparameters make the training process extremely time-consuming and require powerful hardware to run the deep model. Therefore, consering these above issues, Chen and Liu proposed a broad learning system (BLS) \cite{chen2017broad}. BLS belongs to the family of randomized neural networks (RdNNs) \cite{zhang2016survey, quadir2026garfln} as RdNNs incorporate randomness in the topology and learning process of the model, resulting in RdNNs learning with fewer tunable parameters in less time and hardware resources. Thus, BLS has the advantages of a simple architecture and can learn with first speed.  BLS has three primary layers: (i) a feature learning layer, (ii) an enhancement layer, and (iii) an output layer. In feature learning and enhancement layers. In the feature learning and enhancement segment, the input data are transformed into a high-dimensional feature space using a set of random projection mappings. As a result of the feature learning and enhancement segment, a collection of high-dimensional features is produced that effectively extracts important information from the input data.
To train the BLS model, only the output layer parameters
need to be computed by the least-squares method, resulting in no need for backpropagation during training. Further, the BLS has more advantages, like its ability to learn from a small number of training samples without overfitting, the universal approximation capability of BLS \cite{chen2017broad, tanveer2026eda} makes it more promising among researchers, and the adaptable topology of BLS makes it possible to train and update the model effectively in an incremental way \cite{chen2017broad}, and when training data are scarce, BLS may have superior generalization performance than deep learning models \cite{yu2019broad}.

In recent years, a wide range of BLS variants have been proposed to enhance learning efficiency, robustness, and generalization across diverse application domains \cite{TANVEER2026112940}. These efforts have significantly enriched the BLS framework while also revealing its inherent limitations \cite{TANVEER2026108914}. For instance, an iterative gradient-descent-based learning strategy was introduced in \cite{8616379} to update the connections among feature nodes, enhancement nodes, and output nodes, thereby improving adaptability beyond the original closed-form solution. To incorporate human-like reasoning into BLS, Shuang et al. \cite{feng2018fuzzy} proposed a fuzzy BLS that integrates IF–THEN fuzzy rules with the BLS learning architecture. Subsequently, fuzzy BLS (F-BLS) and intuitionistic fuzzy BLS (IF-BLS) models were developed in \cite{sajid2024intuitionistic, tanveer2026fuzzy} to enhance robustness against noise and outliers by explicitly modeling uncertainty and hesitation in training samples.

Further, the graph-embedding intuitionistic fuzzy adaptive BLS (GEIB) \cite{10906533} was introduced to address imbalanced and noisy data by preserving the intrinsic geometric structure of the data manifold. This was achieved by embedding graph-based relations into the BLS framework and assigning adaptive importance scores to training samples using affiliation and non-affiliation degrees from intuitionistic fuzzy theory. More recently, the kernel risk-sensitive mean p-power BLS (KRPBLS) \cite{10902561} was proposed to mitigate the sensitivity of conventional BLS models to label noise by replacing the minimum mean square error criterion with a more robust risk-sensitive optimization objective \cite{quadir2025twin, quadir2025trkm}. Collectively, these studies have substantially improved the flexibility and resilience of BLS models, establishing them as competitive shallow learning architectures in modern machine learning.

Despite these advances, a fundamental limitation remains largely unaddressed. The performance of BLS is highly dependent on the feature learning and enhancement layers, which play a pivotal role in propagating information to the output layer. In standard BLS architectures, hidden-layer representations are generated through random projections. While computationally efficient, this random generation mechanism often fails to adequately capture complex nonlinear structures and higher-order dependencies inherent in real-world data, thereby restricting the representational capacity of the model.

In parallel, complex-valued neural networks (CVNNs) have attracted increasing attention due to their superior ability to model nonlinear relationships and signal coherence \cite{9849162}. Unlike real-valued neural networks (RVNNs), CVNNs offer a natural and mathematically elegant framework for representing physical systems and signals encountered in many engineering applications. Notably, Hirose demonstrated that CVNNs often exhibit smaller generalization errors than their real-valued counterparts, owing to their enhanced capability to capture phase and amplitude interactions \cite{6138313}. Zhang further showed that complex-valued gradient learning with complex step sizes provides theoretical and practical advantages over traditional real-valued learning \cite{7321072}. 

More recently, a complex variant of a shallow RdNN was proposed in \cite{sajid2025rvfl}, where real-valued tabular data are transformed into complex representations and processed using complex weights and activation functions.

Importantly, recent advances in complex statistics indicate that the second-order characteristics of complex-valued data cannot be fully described by the conventional covariance matrix alone. Instead, both the covariance and pseudo-covariance matrices are required to capture the complete second-order statistics \cite{QING2023108792}. To address this issue, complex augmented valued neural networks based on widely linear modeling have been proposed \cite{XU201544}, where the complex signal and its conjugate are jointly processed to exploit all available statistical information.

However, despite the strong theoretical foundations and empirical success of augmented CVNNs, their application has been largely confined to domains where data are inherently complex-valued, such as signal processing and communications. To the best of our knowledge, no complex augmented -valued variant of the Broad Learning System has been developed so far. Consequently, existing BLS models are unable to exploit signal coherence, full second-order statistics, and conjugate information, leaving a significant representational gap, particularly when dealing with nonlinear, noisy, or structurally complex tabular data.

Motivated by this clear gap, this paper proposes a novel Complex Augmented  Broad Learning System (CA-BLS) that integrates complex augmented  representations into both the feature learning and enhancement layers. By incorporating the complex conjugates of these layers, the proposed model is able to capture the complete second-order statistics of the data, including both covariance and pseudo-covariance components, thereby uncovering nonlinear structures that remain inaccessible to real-valued BLS models. While this augmentation significantly enhances representational power, it also introduces additional computational overhead due to the increased dimensionality of complex-valued matrices involved in pseudo-inverse computation.

To address this challenge, we further propose an Efficient Complex Augmented BLS (ECA-BLS), which substantially reduces computational complexity by replacing most complex-valued matrix operations with equivalent real-valued computations, without compromising model performance. Moreover, we rigorously prove the mathematical equivalence between CA-BLS and ECA-BLS, ensuring that the efficiency gains are achieved without any loss of modeling capability.

The primary contributions of this paper can be summarized as follows:

\begin{itemize}
\item A novel mechanism is introduced to transform real-valued tabular datasets into complex-valued representations.
\item A Complex Augmented  Broad Learning System (CA-BLS) is developed, employing complex weights and analytic activation functions to capture full second-order statistics.
\item An Efficient Complex Augmented  BLS (ECA-BLS) is proposed to significantly reduce the computational complexity of CA-BLS.
\item A theoretical proof is provided to establish the equivalence between CA-BLS and ECA-BLS.
\item Comprehensive experimental evaluations demonstrate the superiority of ECA-BLS over recent state-of-the-art models across diverse benchmark datasets.
\end{itemize}

Related works are given in Section S.I. of the Supplementary Material.
\section{Notation}
Let the training dataset be denoted as $\mathcal{Z}_{\text{raw}} = \{(x_j, y_j) \mid x_j \in \mathbb{R}^{B},\, y_j \in \mathbb{R}^{C},\, j = 1, 2, \ldots, D\},$ where $B$ represents the number of features, $C$ denotes the number of output classes, and $D$ is the total number of training samples. The input samples are arranged into a data matrix $\mathbf{Z} \in \mathbb{R}^{D \times B}$, while the corresponding target labels are collected into a matrix $\mathbf{X} \in \mathbb{R}^{D \times C}$.
\section{Proposed method}
\subsection{Transformation of Data from Real Field to Complex Field}
To enable BLS to operate in the complex domain, each real-valued input feature is independently transformed into a complex-valued representation using phase encoding. Given a $(i, j)^{th}$ normalized real-valued input feature of real valued feature matrix $\mathbf{Z}$, $z_{ij} \in [0,1]$, its complex-valued representation is defined as
\begin{equation}
{}^{cx}z_{ij} = e^{(i\pi z_{ij})},
\label{eq:phase_encoding}
\end{equation}
where $i=\sqrt{-1}$, $i = 1, 2, \ldots, D,$ and $ j = 1, 2, \ldots, B.$
This transformation maps real-valued features onto the unit circle in the complex plane, ensuring bounded magnitude and numerical stability. Moreover, phase encoding introduces nonlinear angular separation between samples, thereby enhancing class discriminability without increasing the dimensionality of the input space. After transformation, the real-valued dataset $\mathbf{Z} \in \mathbb{R}^{D \times B}$ is converted into a complex-valued dataset ${}^{cx}\mathbf{Z} \in \mathbb{C}^{D \times B}$, which serves as the input to the proposed complex-valued BLS framework.

\subsection{Proposed Complex Augmented  Broad Learning System (CA-BLS)}
Although complex-valued data representations can effectively exploit phase information and enhance nonlinear modeling capability, they are still insufficient to fully characterize the second-order statistics of complex-valued signals when only covariance information is considered \cite{QING2023108792}. According to the theory of complex-valued random processes, the covariance matrix alone cannot completely describe non-circular complex signals. In such cases, the pseudo-covariance matrix carries additional and essential statistical information.

Phase-encoded complex representations derived from real-valued data are generally non-circular. Consequently, ignoring pseudo-covariance information may lead to suboptimal learning performance. To overcome this limitation, we adopt the principle of complex augmented  modeling and propose an \emph{Complex Augmented  Broad Learning System (CA-BLS)}, which explicitly incorporates conjugate representations to fully exploit complete second-order statistics, thereby improving generalization capability.

Let ${}^{cx}\mathbf{Z} \in \mathbb{C}^{D \times B}$ denote the phase-encoded complex-valued input matrix, and let $\mathbf{X} \in \mathbb{R}^{D \times C}$ represent the corresponding target label matrix. Following the standard BLS architecture, the complex-valued input is mapped to the feature and enhancement layers using randomly generated complex weights, as described below.

\subsubsection{Feature Learning Segment}
Assume that the feature learning layer consists of $a$ feature groups, each containing $b$ nodes. The output of the $i$th feature group is defined as
\begin{equation}
{}^{cx}\mathbf{A}_i =
{}^{cx}\sigma_i\!\left(
{}^{cx}\mathbf{Z}\,{}^{cx}\mathbf{W}_{A_i}
+
{}^{cx}\boldsymbol{\xi}_{A_i}
\right)
\in \mathbb{C}^{D \times b},
\quad i = 1, 2, \ldots, a,
\label{a1}
\end{equation}
where ${}^{cx}\sigma_i(\cdot)$ denotes the complex-valued activation function, ${}^{cx}\mathbf{W}_{A_i} \in \mathbb{C}^{B \times b}$ is a randomly generated weight matrix, and ${}^{cx}\boldsymbol{\xi}_{A_i} \in \mathbb{C}^{D \times b}$ is the corresponding bias matrix. The outputs of all feature groups are concatenated to form
\begin{equation}
{}^{cx}\mathbf{A}^{a}
=
[{}^{cx}\mathbf{A}_1, {}^{cx}\mathbf{A}_2, \ldots, {}^{cx}\mathbf{A}_a]
\in \mathbb{C}^{D \times ab}.
\label{eq:acbls_feature_concat}
\end{equation}

\subsubsection{Enhancement Segment}
The augmented feature matrix ${}^{cx}\mathbf{A}^{a}$ is further projected into the enhancement layer through random complex transformations followed by nonlinear activation. Let $c$ denote the number of enhancement groups, each consisting of $d$ nodes. The output of the $k$th enhancement group is given by
\begin{equation}
{}^{cx}\mathbf{Y}_k =
{}^{cx}\zeta_k\!\left(
{}^{cx}\mathbf{A}^{a}\,{}^{cx}\mathbf{W}_{Y_k}
+
{}^{cx}\boldsymbol{\Gamma}_{Y_k}
\right)
\in \mathbb{C}^{D \times d},
\quad k = 1, 2, \ldots, c,
\label{eq:acbls_enhance}
\end{equation}
where ${}^{cx}\zeta_k(\cdot)$ denotes the complex-valued activation function, ${}^{cx}\mathbf{W}_{Y_k} \in \mathbb{C}^{ab \times d}$ is a randomly generated weight matrix, and ${}^{cx}\boldsymbol{\Gamma}_{Y_k} \in \mathbb{C}^{D \times d}$ is the corresponding bias matrix. The enhancement outputs are concatenated as
\begin{equation}
{}^{cx}\mathbf{Y}^{c}
=
[{}^{cx}\mathbf{Y}_1, {}^{cx}\mathbf{Y}_2, \ldots, {}^{cx}\mathbf{Y}_c]
\in \mathbb{C}^{D \times cd}.
\label{eq:acbls_enhance_concat}
\end{equation}

\subsubsection{Output Segment and Augmented Modeling}
The feature and enhancement matrices are concatenated to form the complex hidden-layer output matrix
\begin{equation}
\mathbf{H}
=
[{}^{cx}\mathbf{A}^{a}, {}^{cx}\mathbf{Y}^{c}]
\in \mathbb{C}^{D \times (ab + cd)}.
\label{eq:acbls_H}
\end{equation}
To capture the complete second-order statistics, including both covariance and pseudo-covariance information, we construct the augmented hidden matrix as
\begin{equation}
\mathbf{H}_{a}
=
[\mathbf{H}, \mathbf{H}^{*}]
\in \mathbb{C}^{D \times 2(ab + cd)},
\label{eq:acbls_Ha}
\end{equation}
where $(\cdot)^{*}$ denotes complex conjugation.
The output weights of CA-BLS are obtained by solving the regularized least-squares problem
\begin{equation}
\arg\min_{\boldsymbol{\beta}_{a}}
\frac{1}{2}
\|\mathbf{H}_{a}\boldsymbol{\beta}_{a} - \mathbf{X}\|_{2}^{2}
+
\frac{\lambda_{a}}{2}
\|\boldsymbol{\beta}_{a}\|_{2}^{2},
\label{eq:acbls_obj}
\end{equation}
where $\boldsymbol{\beta}_{a} \in \mathbb{C}^{2(ab+cd) \times C}$ denotes the output weight matrix and $\lambda_{a}$ is a regularization parameter. The closed-form solution is given by
\begin{equation}
\boldsymbol{\beta}_{a}
=
(\mathbf{H}_{a}^{H}\mathbf{H}_{a} + \lambda_{a}\mathbf{I})^{-1}
\mathbf{H}_{a}^{H}\mathbf{X},
\label{eq:acbls_solution}
\end{equation}
where $(\cdot)^{H}$ denotes the Hermitian transpose and $\mathbf{I}$ is an identity matrix of appropriate dimension.

\subsection{Efficient Complex Augmented  Broad Learning System (ECA-BLS)}
Although CA-BLS significantly enhances modeling capability by exploiting complete second-order statistics, the augmentation process doubles the number of hidden nodes and requires complex-valued matrix inversion, leading to increased computational cost. To address this issue, we propose a computationally efficient reformulation termed as \emph{Efficient Complex Augmented  Broad Learning System (ECA-BLS)}. Figure~\ref{fig:ECA_bls} illustrates the architecture of ECA-BLS. The real-valued input is first transformed into the complex domain using phase encoding. Complex-valued feature and enhancement representations are then constructed following the BLS framework. Instead of explicitly forming the complex augmented matrix, ECA-BLS employs a real-valued augmentation strategy to implicitly capture conjugate information while significantly reducing computational complexity. Given the matrix $\mathbf{H}$ in \eqref{eq:acbls_H}, its real-valued augmented representation is defined as
\begin{equation}
\mathbf{H}_{r}
=
[\Re(\mathbf{H}), \Im(\mathbf{H})]
\in \mathbb{R}^{D \times 2(ab + cd)},
\label{eq:ECAbls_Hr}
\end{equation}
where $\Re(\cdot)$ and $\Im(\cdot)$ denote the real and imaginary parts, respectively. The output weights of ECA-BLS are obtained by solving
\begin{equation}
\arg\min_{\boldsymbol{\beta}_{r}}
\frac{1}{2}
\|\mathbf{H}_{r}\boldsymbol{\beta}_{r} - \mathbf{X}\|_{2}^{2}
+
\frac{\lambda_{r}}{2}
\|\boldsymbol{\beta}_{r}\|_{2}^{2},
\end{equation}
where $\boldsymbol{\beta}_{r} \in \mathbb{R}^{2(ab+cd) \times C}$ and $\lambda_{r}$ is a regularization parameter. The closed-form solution is
\begin{equation}
\boldsymbol{\beta}_{r}
=
(\mathbf{H}_{r}^{T}\mathbf{H}_{r} + \lambda_{r}\mathbf{I})^{-1}
\mathbf{H}_{r}^{T}\mathbf{X}.
\label{eq:ECAbls_solution}
\end{equation}
Algorithm and complexity of the proposed ECA-BLS are given in given in Section S.III. and S.IV. of Supplementary Material, respectively.

\subsection{Theoretical Justification of Equivalence between Proposed CA-BLS and ECA-BLS} The theoretical proof of equivalence between Proposed CA-BLS and ECA-BLS is given in Section S.II. of Supplementary Material.
\begin{figure}
    \centering
    \includegraphics[width=.8\linewidth]{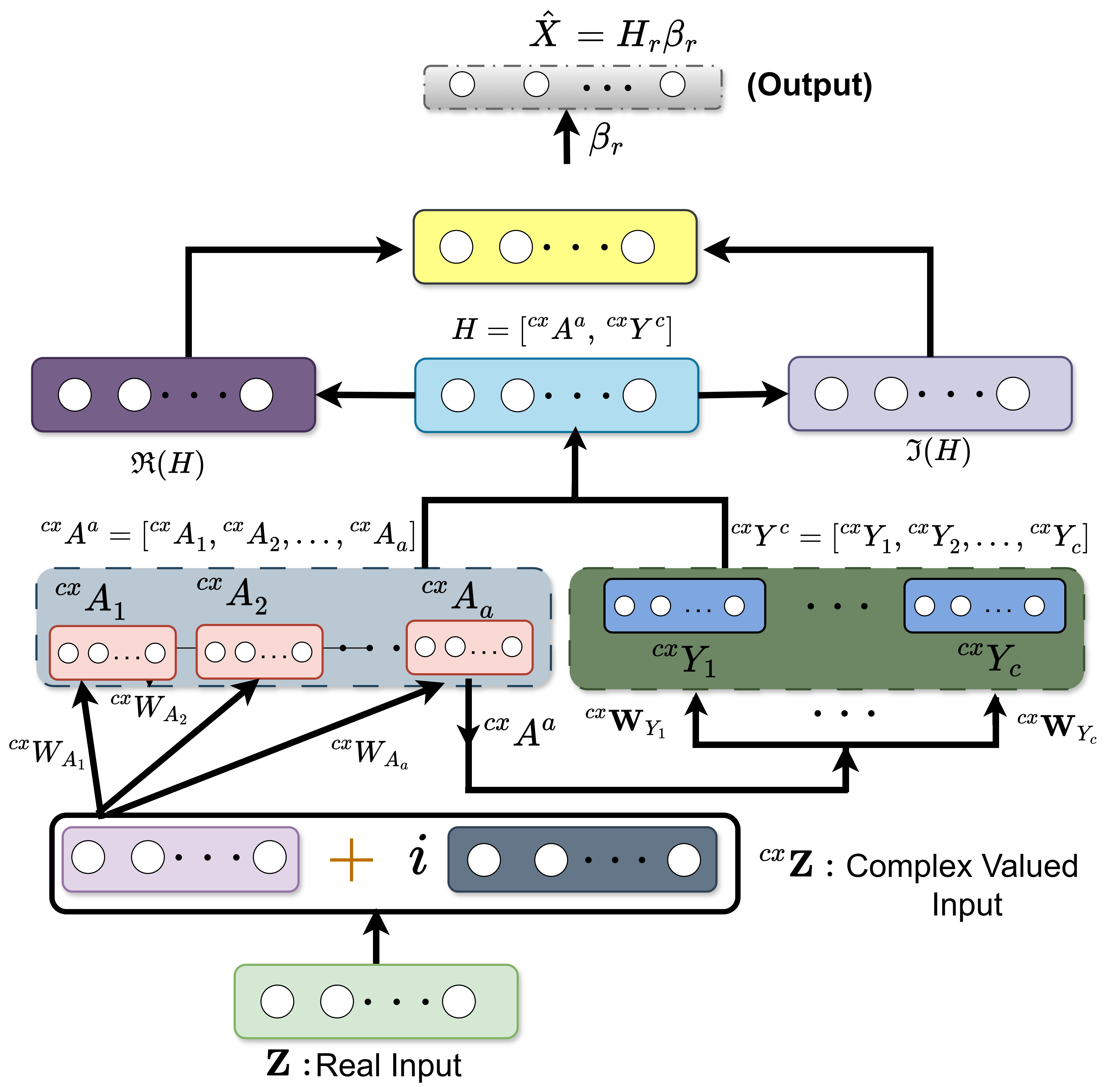}
    \caption{Architecture of ECA-BLS.}
    \label{fig:ECA_bls}
\end{figure}

The theoretical proof establishes a fundamental duality between the complex and real formulations:
\begin{itemize}
    \item {Mathematical Isomorphism:} 
    CA-BLS and ECA-BLS are algebraically isomorphic. The transformation matrix $\mathbf{V}$ defines a bijective mapping between the complex augmented feature space and the real expanded feature space, proving that both models optimize the exact same objective function in different coordinate systems.

    \item {Preservation of Phase Information:} 
    ECA-BLS is not a simplification that discards complex characteristics. Instead, it implicitly captures the full phase-amplitude relationships of the widely linear model by encoding them within the cross-correlations of the separated real and imaginary channels.

    \item {Zero Theoretical Loss:} 
    The proof guarantees that the output equivalence is exact, not approximate. Provided the regularization scaling $\lambda_a = 2\lambda_{r}$ is maintained, the real-valued implementation achieves the identical decision boundary as the complex-valued formulation with zero loss in accuracy.

    \item {Computational Efficiency:} 
    While theoretically equivalent, ECA-BLS is computationally superior. By avoiding complex arithmetic where a single multiplication requires four real multiplications and two additions, ECA-BLS leverages highly optimized standard real-valued linear algebra, significantly reducing computational cost and memory overhead.
\end{itemize}

Based on this theoretical equivalence between CA-BLS and ECA-BLS, and the efficiency of ECA-BLS, we have conducted the experiments in the ECA-BLS to compare the performance and statistical significance of our proposed models.


\begin{table*}[!htbp]
\centering
\caption{Classification accuracies along with average accuracy and average rank for ECA-BLS model evaluated against baseline models on 26 UCI and KEEL datasets.}
\label{tab:multi_small}
\renewcommand{\arraystretch}{.65} 
\resizebox{\textwidth}{!}{%
\begin{tabular}{lcccccccc}
 \hline
\textbf{Dataset} $\downarrow$  \textbf{\textbar{} Model} $\rightarrow$ &
  BLS\cite{chen2017broad} &
  H-ELM\cite{tang2015extreme} &
  GEIB \cite{10906533}   &
  F-BLS \cite{sajid2024intuitionistic} &
  IF-BLS \cite{sajid2024intuitionistic} &

  KRP-BLS \cite{10902561} &
  \textbf{ECA-BLS}$^{\dagger}$ \\
\hline
acute\_inflammation & 88.8889 & 91.6667 & {100} & 86.1111 & {100} & {100} & {100} \\
bank & 88.5041 & 88.8643 & 71.4602 & 85.409 & {89.1673} & 88.1356 & 88.1356 \\
brwisconsin & 92.1951 & 87.2549 & {96.5686} & 89.7561 & 63.4146 & 92.6829 & 87.8049 \\
chess\_krvkp & 91.9708 & 80.65 & 84.9687 & 85.09 & 86.1314 & 85.4015 & {93.8478} \\
cleve & 82.2222 & 75.2809 & 83.1461 & 56.6667 & 82.2222 & 80 & {85.5556} \\
ecoli-0-1\_vs\_5 & 94.4444 & 88.89 & {97.2222} & 90.94 & 51.3889 & 62.5 & 88.8889 \\
ecoli-0-1-4-7\_vs\_2-3-5-6 & 90.099 & 87.13 & 75 & {90.1} & 90.099 & 90.099 & 88.1188 \\
ecoli-0-1-4-7\_vs\_5-6 & 93 & 93 & 74.7475 & 94 & 91 & {95} & 90 \\
ecoli-0-6-7\_vs\_5 & 60.6061 & 90.9091 & 90.9091 & 66.6667 & {98.4848} & 93.9394 & 90.9091 \\
haber & 78.2609 & {82.61} & 70.3297 & 76.26 & 79.3478 & 78.2609 & 79.3478 \\
heart\_hungarian & {80.8989} & 64.7727 & 80.6818 & 59.5506 & 79.7753 & 74.1573 & {80.8989} \\
heart-stat & 87.6543 & 80.21 & 81.4815 & 78.89 & 85.1852 & {88.8889} & 82.716 \\
ionosphere & 86.7925 & 78.0952 & {88.5714} & 81.1321 & 86.7925 & 85.8491 & 83.0189 \\
led7digit-0-2-4-5-6-7-8-9\_vs\_1 & 73.6842 & 80.23 & {91.6667} & 80.45 & 47.3684 & 75.188 & 90.9774 \\
mammographic & 79.5848 & 78.94 & 80.9028 & 78.89 & 82.0069 & 79.2388 & {83.045} \\
monk1 & 41.3174 & 46.988 & 37.3494 & 49.1018 & 46.1078 & {55.0898} & 54.491 \\
musk\_1 & 76.9231 & 75.15 & 69.0141 & 78.32 & 79.7203 & 80.4196 & {90.2098} \\
oocytes\_merluccius\_nucleus\_4d & 56.6775 & 70.18 & 72.549 & 72.2 & {75.57} & 74.9186 & 73.6156 \\
ozone & 85.5834 & 96.58 & 83.6842 & 96.58 & 80.1564 & 83.1731 & {96.5834} \\
ripley & 89.3333 & 60.8 & {92} & 84.2667 & 88.8 & 88.5333 & 70.1333 \\
shuttle-6\_vs\_2-3 & 94.2029 & 95.65 & 92.2754 & {100} & {100} & 92 & 97.1014 \\
spambase & 82.042 & 69.27 & {87.7536} & 65.89 & 82.5455 & 70.6731 & 86.097 \\
spectf & 81.4815 & 75.75 & 67.5 & 80.32 & 75.3086 & 77.7778 & {82.716} \\
vertebral\_column\_2clases & 65.5591 & 69.89 & {72.043} & 65.56 & 64.2308 & 64.4086 & 69.8925 \\
wpbc & 71.1864 & {77.97} & 75.1724 & 76.27 & 70.5932 & 59.322 & 77.9661 \\
yeast-2\_vs\_4 & 83.871 & 82.81 & 83.4416 & 82.39 & {94.8387} & 87.7419 & 85.8065 \\
\hline
{\textbf{Average Accuracy}} & 80.6532 & 79.5978 & 80.7861 & 78.8773 & 79.6252 & 80.9 & \textbf{84.5337} \\
\hline
{\textbf{Average rank}} & 3.9423 & {4.7692} & 4.0577 & 4.6923 & 3.75 & 3.8654 & \textbf{2.9231}\\

\hline
\end{tabular}%
}
\end{table*}


\section*{S.V. Experimental Setup, Compared Models and Datasets}
Detailed experimental setup, compared models, and datasets are discussed in Section S.V. in Supplementary Material.
\subsection{Evaluation on UCI and KEEL Dataset}  
Table~\ref{tab:multi_small} summarizes the experimental results obtained on 26 benchmark datasets from the KEEL and UCI repositories, reporting classification accuracy for each method along with the corresponding average accuracy and average rank. The complete set of results, including the optimal hyperparameter configurations for the proposed model across all datasets, is provided in the Supplementary Material (Table~S.II). Across all evaluated datasets and competing methods, the proposed ECA-BLS consistently demonstrates superior performance. In particular, ECA-BLS achieves the highest average classification accuracy of 84.5337\%, clearly outperforming all baseline and state-of-the-art models considered in this study. The closest competitors, KRPBLS and GEIB, attain average accuracies of 80.9\% and 80.7861\%, respectively, resulting in a substantial performance margin in favor of ECA-BLS. All remaining state-of-the-art methods achieve average accuracies below 80\%, further emphasizing the effectiveness of the proposed approach.

These results provide strong empirical evidence that the proposed ECA-BLS offers enhanced generalization capability and robustness across a wide range of data distributions.
\noindent
\textbf{Statistical rank:}  
Average accuracy may be biased by exceptional performance on individual datasets and does not always reflect a model’s overall robustness. To obtain a fairer evaluation, we adopt a ranking-based assessment \cite{demvsar2006statistical}, where models are evaluated independently on each dataset. Lower ranks are assigned to better-performing models, while higher ranks indicate poorer performance. Given \(u\) models evaluated over \(v\) datasets, let \(R_r^s\) denote the rank of the \(r^{th}\) model on the \(s^{th}\) dataset. The average rank of the \(r^{th}\) model is computed as $R_r = \frac{1}{v} \sum_{s=1}^{v} R_r^s,$ which provides an aggregated and unbiased measure of comparative performance across all datasets. The proposed ECA-BLS achieves the lowest average rank of 2.9231 among all compared methods, indicating the most consistent performance across the evaluated datasets. The second and third best positions are attained by IF-BLS and KRPBLS, respectively. In contrast, recent advanced baselines such as GEIB and F-BLS obtain noticeably higher average ranks of 4.0577 and 4.6923, respectively. Likewise, classical models, including H-ELM and BLS, exhibit inferior rankings of 4.7692 and 3.9423, further highlighting the competitive advantage of the proposed approach. Overall, the ranking-based analysis provides compelling statistical evidence that ECA-BLS delivers superior and more stable performance than both classical and state-of-the-art competitors, thereby confirming its strong generalization capability across diverse benchmark datasets.

\noindent
\textbf{Friedman test:}  
The Friedman test \cite{friedman1940comparison} is employed to statistically examine whether significant performance differences exist among the compared models. Under the null hypothesis, all models are assumed to exhibit equivalent performance, corresponding to equal average ranks. The Friedman test statistic follows a chi-squared distribution, denoted by $\chi^2_F$, with $(u-1)$ degrees of freedom, and is computed as$\chi^2_F = \frac{12v}{u(u+1)}
\left(
\sum_{i=1}^{u} R_i^2 - \frac{u(u+1)^2}{4}
\right),$ where $u$ is the number of models, $v$ is the number of datasets, and $R_i$ denotes the average rank of the $i^{th}$ model. To obtain a more accurate test under finite sample sizes, the Friedman statistic is further transformed into an $F$-statistic given by $F_F =
\frac{(v-1)\chi^2_F}{v(u-1)-\chi^2_F},$ which follows an $F$-distribution with $(u-1)$ and $(v-1)(u-1)$ degrees of freedom. In our experiments, considering $u=7$ models evaluated over $v=26$ datasets, the computed statistics are $\chi^2_F = 12.9142$ and $F_F = 2.2564$. At a significance level of $5\%$, the critical value of the $F$-distribution with $(6,150)$ degrees of freedom is $2.1595$. Since $F_F > 2.1595$, the null hypothesis is rejected, indicating statistically significant performance differences among the compared models. Sections S.VII., S.VIII., and S.IX of Supplementary Material discussed the Win--Tie--Loss (W--T--L) sign test, sensitivity analysis, and evaluation of the proposed ECA-BLS under contaminated Gaussian noise, respectively.

\section{Conclusion and Future Work}
\label{Sec:Conclusion}

This paper fundamentally advances Broad Learning Systems by introducing complex augmented modeling into the BLS framework for the first time. Unlike existing BLS variants that focus on architectural or robustness enhancements within the real domain, the proposed approach overcomes the intrinsic inability of real-valued BLS to exploit complete second-order statistical information, thereby closing a long-standing gap in the literature. By employing phase-encoded complex and complex augmented representations within a widely linear learning formulation, the proposed CA-BLS jointly leverages covariance and pseudo-covariance information to capture nonlinear structures and coherence patterns that are inaccessible to conventional BLS models. To ensure computational efficiency and scalability, an equivalent real-domain formulation, termed ECA-BLS, is rigorously derived and proven to preserve the exact decision function of CA-BLS. Extensive experiments on diverse benchmark datasets, supported by Friedman and win–tie–loss statistical analyses, consistently demonstrate the superior generalization, robustness, and stability of ECA-BLS over classical BLS and recent state-of-the-art variants. Together, the theoretical guarantees and empirical evidence establish that complex augmented learning is a missing foundational component of BLS research. Consequently, this work elevates BLS from a purely real-domain randomized model to a more expressive, statistically complete, and practically deployable learning paradigm.
\bibliographystyle{IEEEtran}
\bibliography{ref}

@article{sepehri2024hierarchical,
  title={Hierarchical training of deep neural networks using early exiting},
  author={Sepehri, Yamin and Pad, Pedram and Y{\"u}z{\"u}g{\"u}ler, Ahmet Caner and Frossard, Pascal and Dunbar, L Andrea},
  journal={IEEE Transactions on Neural Networks and Learning Systems},
  volume={36},
  number={4},
  pages={6271--6285},
  year={2024},
  publisher={IEEE}
}

@ARTICLE{10413606,
  author={De Santis, Enrico and Martino, Alessio and Rizzi, Antonello},
  journal={IEEE Transactions on Pattern Analysis and Machine Intelligence}, 
  title={{Human Versus Machine Intelligence: Assessing Natural Language Generation Models Through Complex Systems Theory}}, 
  year={2024},
  volume={46},
  number={7},
  pages={4812-4829},
  doi={10.1109/TPAMI.2024.3358168}}

@article{chen2017broad,
  title={{Broad learning system: An effective and efficient incremental learning system without the need for deep architecture}},
  author={Chen, CL Philip and Liu, Zhulin},
  journal={IEEE transactions on neural networks and learning systems},
  volume={29},
  number={1},
  pages={10--24},
  year={2017},
  publisher={IEEE}
}

@article{zhang2016survey,
  title={A survey of randomized algorithms for training neural networks},
  author={Zhang, Le and Suganthan, Ponnuthurai N},
  journal={Information Sciences},
  volume={364},
  pages={146--155},
  year={2016},
  publisher={Elsevier}
}

@article{yu2019broad,
  title={Broad convolutional neural network based industrial process fault diagnosis with incremental learning capability},
  author={Yu, Wanke and Zhao, Chunhui},
  journal={IEEE Transactions on Industrial Electronics},
  volume={67},
  number={6},
  pages={5081--5091},
  year={2019},
  publisher={IEEE}
}

@INPROCEEDINGS{8616379,
  author={Feng, Shuang and Chen, C. L. Philip},
  booktitle={2018 IEEE International Conference on Systems, Man, and Cybernetics (SMC)}, 
  title={{Broad Learning System for Control of Nonlinear Dynamic Systems}}, 
  year={2018},
  volume={},
  number={},
  pages={2230-2235},
  doi={10.1109/SMC.2018.00383}}

@article{feng2018fuzzy,
  title={{Fuzzy broad learning system: A novel neuro-fuzzy model for regression and classification}},
  author={Feng, Shuang and Chen, CL Philip},
  journal={IEEE transactions on cybernetics},
  volume={50},
  number={2},
  pages={414--424},
  year={2018},
  publisher={IEEE}
}

@article{sajid2024intuitionistic,
  title={{Intuitionistic fuzzy broad learning system: Enhancing robustness against noise and outliers}},
  author={Sajid, M and Malik, Ashwani Kumar and Tanveer, Muhammad},
  journal={IEEE Transactions on Fuzzy Systems},
  volume={32},
  number={8},
  pages={4460--4469},
  year={2024},
  publisher={IEEE}
}

@ARTICLE{10906533,
  author={Chen, Wuxing and Yang, Kaixiang and Yu, Zhiwen and Nie, Feiping and Chen, C. L. Philip},
  journal={IEEE Transactions on Fuzzy Systems}, 
  title={{Adaptive Broad Network With Graph-Fuzzy Embedding for Imbalanced Noise Data}}, 
  year={2025},
  volume={33},
  number={6},
  pages={1949-1962},
  doi={10.1109/TFUZZ.2025.3543369}}

@ARTICLE{10902561,
  author={Deng, Wu and Shen, Jiuru and Ding, Jianming and Zhao, Huimin},
  journal={IEEE Internet of Things Journal}, 
  title={{Robust Dual-Model Collaborative Broad Learning System for Classification Under Label Noise Environments}}, 
  year={2025},
  volume={12},
  number={12},
  pages={21055-21067},
  doi={10.1109/JIOT.2025.3545741}}

@ARTICLE{9849162,
  author={Lee, ChiYan and Hasegawa, Hideyuki and Gao, Shangce},
  journal={IEEE/CAA Journal of Automatica Sinica}, 
  title={{Complex-Valued Neural Networks: A Comprehensive Survey}}, 
  year={2022},
  volume={9},
  number={8},
  pages={1406-1426},
  doi={10.1109/JAS.2022.105743}}

@ARTICLE{6138313,
  author={Hirose, Akira and Yoshida, Shotaro},
  journal={IEEE Transactions on Neural Networks and Learning Systems}, 
  title={{Generalization Characteristics of Complex-Valued Feedforward Neural Networks in Relation to Signal Coherence}}, 
  year={2012},
  volume={23},
  number={4},
  pages={541-551},
  doi={10.1109/TNNLS.2012.2183613}}

@ARTICLE{7321072,
  author={Zhang, Huisheng and Mandic, Danilo P.},
  journal={IEEE Transactions on Neural Networks and Learning Systems}, 
  title={{Is a Complex-Valued Stepsize Advantageous in Complex-Valued Gradient Learning Algorithms?}}, 
  year={2016},
  volume={27},
  number={12},
  pages={2730-2735},
  doi={10.1109/TNNLS.2015.2494361}}

@article{QING2023108792,
title = {Performance analysis of the augmented complex-valued least mean kurtosis algorithm},
journal = {Signal Processing},
volume = {203},
pages = {108792},
year = {2023},
issn = {0165-1684},
doi = {https://doi.org/10.1016/j.sigpro.2022.108792},
author = {Zhu Qing and Jingen Ni and Zhe Li and Engin Cemal Mengüç and Jie Chen and Danilo P. Mandic}
}

@article{XU201544,
title = {Convergence analysis of an augmented algorithm for fully complex-valued neural networks},
journal = {Neural Networks},
volume = {69},
pages = {44-50},
year = {2015},
issn = {0893-6080},

author = {Dongpo Xu and Huisheng Zhang and Danilo P. Mandic}
}

@inproceedings{sajid2025rvfl,
  title={{RVFL-X: A Novel Randomized Network Based on Complex Transformed Real-Valued Tabular Datasets}},
  author={Sajid, M and Akhtar, Mushir and Quadir, A and Tanveer, M},
  booktitle={2025 International Joint Conference on Neural Networks (IJCNN)},
  pages={1--8},
  year={2025},
  organization={IEEE}
}

@article{tang2015extreme,
  title={Extreme learning machine for multilayer perceptron},
  author={Tang, Jiexiong and Deng, Chenwei and Huang, Guang-Bin},
  journal={IEEE transactions on neural networks and learning systems},
  volume={27},
  number={4},
  pages={809--821},
  year={2015},
  publisher={Ieee}
}

@article{dua2017uci,
  title={{{UCI} machine learning repository.}},
  author={Dua, Dheeru and Graff, Casey},
journal={Available: http://archive.ics.uci.edu/ml},
  year={2017},
  publisher={Beijing China}
}

@article{derrac2015keel,
  title={{Keel data-mining software tool: Data set repository, integration of algorithms and experimental analysis framework}},
  author={Derrac, J and Garcia, S and Sanchez, L and Herrera, F},
  journal={J. Mult. Valued Logic Soft Comput},
  volume={17},
  pages={255--287},
  year={2015}
}

@article{demvsar2006statistical,
  title={{Statistical comparisons of classifiers over multiple data sets}},
  author={Dem{\v{s}}ar, Janez},
  journal={The Journal of Machine Learning Research},
  volume={7},
  pages={1--30},
  year={2006},
  publisher={JMLR. org}
}

@article{friedman1940comparison,
  title={{A comparison of alternative tests of significance for the problem of m rankings}},
  author={Friedman, Milton},
  journal={The Annals of Mathematical Statistics},
  volume={11},
  number={1},
  pages={86--92},
  year={1940},
  publisher={JSTOR}
}

@article{TANVEER2026112940,
title = {{BLS-CIL: Class Imbalance Broad Learning System via Dual Weighting and Layer Trimming}},
journal = {Pattern Recognition},
volume = {174},
pages = {112940},
year = {2026},
issn = {0031-3203},

author = {M. Tanveer and A. Mishra and M. Sajid and A. Quadir}
}

@article{TANVEER2026108914,
title = {{Fuzzy-driven broad learning system with class probability and density awareness for multi-view data}},
journal = {Neural Networks},
volume = {201},
pages = {108914},
year = {2026},
issn = {0893-6080},

author = {M. Tanveer and M. Pathak and M. Sajid and A. Quadir and  Priyamvada}
}

@article{quadir2026hypergraph,
  title={Hypergraph neural network with state space models for node classification},
  author={Quadir, Abdul and Tanveer, M},
  journal={Engineering Applications of Artificial Intelligence},
  volume={163},
  pages={112922},
  year={2026},
  publisher={Elsevier}
}

@article{quadir2026garfln,
  title={{GARFLN: Geodesic Adaptive Riemannian Functional Link Network}},
  author={Quadir, A and Tanveer, M},
  journal={Pattern Recognition},
  pages={114163},
  year={2026},
  publisher={Elsevier}
}

@article{tanveer2026eda,
  title={{EDA-OCBLS: An error-distribution aware one-class broad learning system for anomaly detection}},
  author={Tanveer, M and Mishra, A and Quadir, A and Sajid, M},
  journal={Neural Networks},
  pages={109177},
  year={2026},
  publisher={Elsevier}
}

@article{tanveer2026fuzzy,
  title={Fuzzy-driven broad learning system with class probability and density awareness for multi-view data},
  author={Tanveer, M and Pathak, M and Sajid, M and Quadir, Abdul and others},
  journal={Neural Networks},
  pages={108914},
  year={2026},
  publisher={Elsevier}
}

@inproceedings{quadir2025twin,
  title={Twin restricted kernel machines for multiview classification},
  author={Quadir, Abdul and Sajid, M and Akhtar, Mushir and Tanveer, Muhammad},
  booktitle={2025 International Joint Conference on Neural Networks (IJCNN)},
  pages={1--8},
  year={2025},
  organization={IEEE}
}

@article{quadir2025trkm,
  title={{TRKM: Twin restricted kernel machines for classification and regression}},
  author={Quadir, Abdul and Tanveer, Muhammad},
  journal={Neural Networks},
  pages={108449},
  year={2025},
  publisher={Elsevier}
}

\clearpage
\section*{Supplementary Material}

\section*{S.I. Related Works}\label{Sec:Related_work}
This section provides a concise yet comprehensive review of the Broad Learning System (BLS), together with its mathematical formulation.

\renewcommand{\thefigure}{S.1}
\begin{figure}[h]
    \centering
    \includegraphics[width=\linewidth]{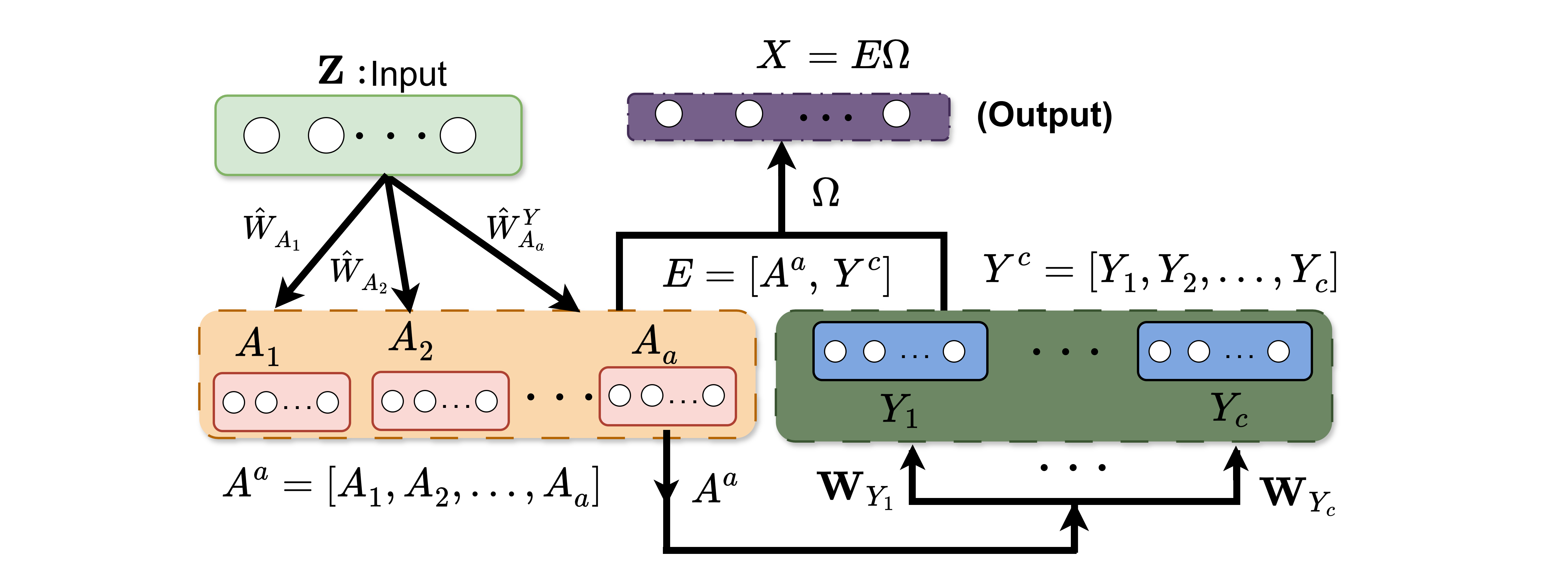}
    \caption{Architecture of the Broad Learning System (BLS).}
    \label{fig:bls}
\end{figure}

\subsection{Broad Learning System (BLS)}
The Broad Learning System (BLS), proposed by Chen \emph{et al.}, is a flat network architecture designed to achieve fast learning with strong generalization capability. Unlike deep neural networks, BLS expands the network width rather than depth and consists of four main components: the input layer, feature learning layer, enhancement layer, and output layer. The connections between successive layers are randomly generated, while the output weights are computed analytically using a closed-form solution. The overall architecture of BLS is illustrated in Fig.~\ref{fig:bls}. The mathematical formulation of each component is described below.

\subsubsection{Feature Learning Segment}
Assume that the feature learning layer is composed of $a$ feature groups, each containing $b$ nodes. The output of the $i$th feature group is defined as
\begin{equation}
\mathbf{A}_i = \sigma_i\!\left(\mathbf{Z}\mathbf{W}_{A_i} + \boldsymbol{\xi}_{A_i}\right) \in \mathbb{R}^{D \times b}, 
\quad i = 1, 2, \ldots, a,
\end{equation}
where $\sigma_i(\cdot)$ denotes the nonlinear activation function, $\mathbf{W}_{A_i} \in \mathbb{R}^{B \times b}$ is a randomly generated weight matrix, and $\boldsymbol{\xi}_{A_i} \in \mathbb{R}^{D \times b}$ is the corresponding bias matrix. The outputs of all feature groups are concatenated to form the augmented feature matrix
\begin{equation}
\mathbf{A}^{a} = [\mathbf{A}_1, \mathbf{A}_2, \ldots, \mathbf{A}_a] \in \mathbb{R}^{D \times ab}.
\end{equation}

\subsubsection{Enhancement Segment}
To further enrich the feature representation, the augmented feature matrix $\mathbf{A}^{a}$ is projected into an enhancement layer through random mappings followed by nonlinear activation. Let $c$ denote the number of enhancement groups, with each group consisting of $d$ nodes. The output of the $k$th enhancement group is given by
\begin{equation}
\mathbf{Y}_k = \zeta_k\!\left(\mathbf{A}^{a}\mathbf{W}_{Y_k} + \boldsymbol{\Gamma}_{Y_k}\right) \in \mathbb{R}^{D \times d}, 
\quad k = 1, 2, \ldots, c,
\end{equation}
where $\zeta_k(\cdot)$ is a nonlinear activation function, $\mathbf{W}_{Y_k} \in \mathbb{R}^{ab \times d}$ is a randomly generated weight matrix, and $\boldsymbol{\Gamma}_{Y_k} \in \mathbb{R}^{D \times d}$ is the bias matrix. The outputs of all enhancement groups are concatenated to form
\begin{equation}
\mathbf{Y}^{c} = [\mathbf{Y}_1, \mathbf{Y}_2, \ldots, \mathbf{Y}_c] \in \mathbb{R}^{D \times cd}.
\end{equation}
\subsubsection{Output Segment}
The output layer receives the concatenation of the feature and enhancement matrices. The network output is computed as
\begin{equation}
\mathbf{O} = [\mathbf{A}^{a}, \mathbf{Y}^{c}]\,\boldsymbol{\Omega} = \mathbf{E}\boldsymbol{\Omega},
\end{equation}
where
\begin{equation}
\mathbf{E} = [\mathbf{A}^{a}, \mathbf{Y}^{c}] \in \mathbb{R}^{D \times (ab + cd)}
\end{equation}
is the final hidden-layer output matrix, and $\boldsymbol{\Omega} \in \mathbb{R}^{(ab + cd) \times C}$ denotes the output weight matrix. The optimal output weights are obtained analytically using the Moore–Penrose pseudoinverse as
\begin{equation}
\boldsymbol{\Omega} = \mathbf{E}^{\dagger}\mathbf{X},
\end{equation}
where $\mathbf{E}^{\dagger}$ represents the pseudoinverse of $\mathbf{E}$.

\section*{S.II Theoretical proof on the equivalence of the proposed CA-BLS and ECA-BLS}
\begin{theorem}
Suppose the training samples are distinct and the same hidden representation is used. 
Let $\mathbf{\beta}_a$ denote the output weight matrix obtained from CA-BLS with regularization parameter $\lambda_a$. 
Let $\mathbf{\beta}_{r}$ denote the output weight matrix obtained from ECA-BLS with regularization parameter $\lambda_{r}$. 
If $\lambda_a = 2\lambda_{r}$, then the outputs of the two models are identical, i.e.,
\begin{equation}
\mathbf{H}_a \mathbf{\beta}_a = \mathbf{H}_{r} \mathbf{\beta}_{r}.
\end{equation}
\end{theorem}

\begin{proof}
Let $\mathbf{H} \in \mathbb{C}^{N \times M}$ be the complex-valued hidden-layer output matrix.

\paragraph{Augmented representations.}
The CA-BLS augmented matrix is defined as
\begin{equation}
\mathbf{H}_a = [\,\mathbf{H},\, \mathbf{H}^*\,] \in \mathbb{C}^{D \times 2(ab+cd)},
\end{equation}
while the ECA-BLS real augmented matrix is given by
\begin{equation}
\mathbf{H}_{r} = [\,\Re(\mathbf{H}),\, \Im(\mathbf{H})\,] \in \mathbb{R}^{D \times 2(ab+cd)}.
\end{equation}

\paragraph{Linear transformation between augmented matrices.}
Define the transformation matrix
\begin{equation}
\mathbf{V} =
\begin{bmatrix}
\mathbf{I} & \mathbf{I} \\
i\mathbf{I} & -i\mathbf{I}
\end{bmatrix}
\in \mathbb{C}^{2(ab+cd) \times 2(ab+cd)},
\end{equation}
where $\mathbf{I}$ denotes the $M \times M$ identity matrix.
Using $\mathbf{H} = \Re(\mathbf{H}) + i\Im(\mathbf{H})$ and 
$\mathbf{H}^* = \Re(\mathbf{H}) - i\Im(\mathbf{H})$, it follows that
\begin{equation}
\mathbf{H}_a = \mathbf{H}_{r}\mathbf{V}.
\label{eq:Ha_relation}
\end{equation}

It is straightforward to verify that
\begin{equation}
\mathbf{V}\mathbf{V}^H = 2\mathbf{I} \implies \mathbf{V}^{-1} = \frac{1}{2}\mathbf{V}^H.
\label{eq:V_property}
\end{equation}

\paragraph{Output weight solutions.}
The CA-BLS output weights are computed as
\begin{equation}
\mathbf{\beta}_a = 
(\mathbf{H}_a^H \mathbf{H}_a + \lambda_a \mathbf{I})^{-1} 
\mathbf{H}_a^H \mathbf{x}.
\end{equation}
Substituting $\mathbf{H}_a = \mathbf{H}_{r}\mathbf{V}$ and using the property $\mathbf{H}_{r}^H = \mathbf{H}_{r}^T$ (since $\mathbf{H}_{r}$ is real), we obtain:
\begin{equation}
\mathbf{\beta}_a = 
(\mathbf{V}^H \mathbf{H}_{r}^T \mathbf{H}_{r} \mathbf{V} + \lambda_a \mathbf{I})^{-1} 
\mathbf{V}^H \mathbf{H}_{r}^T \mathbf{x}.
\end{equation}

To simplify the inverse term, we use the identity $\mathbf{I} = \frac{1}{2}\mathbf{V}^H\mathbf{V}$ derived from \eqref{eq:V_property}. We substitute $\lambda_a \mathbf{I} = \frac{\lambda_a}{2} \mathbf{V}^H \mathbf{V}$:
\begin{equation}
\mathbf{\beta}_a = 
\left( \mathbf{V}^H \mathbf{H}_{r}^T \mathbf{H}_{r} \mathbf{V} + \mathbf{V}^H \left(\frac{\lambda_a}{2}\mathbf{I}\right) \mathbf{V} \right)^{-1} 
\mathbf{V}^H \mathbf{H}_{r}^T \mathbf{x}.
\end{equation}
Factoring out $\mathbf{V}^H$ and $\mathbf{V}$:
\begin{equation}
\mathbf{\beta}_a = 
\left[ \mathbf{V}^H \left( \mathbf{H}_{r}^T \mathbf{H}_{r} + \frac{\lambda_a}{2}\mathbf{I} \right) \mathbf{V} \right]^{-1} 
\mathbf{V}^H \mathbf{H}_{r}^T \mathbf{x}.
\end{equation}
Using the inverse product property $(ABC)^{-1} = C^{-1}B^{-1}A^{-1}$:
\begin{equation}
\mathbf{\beta}_a = 
\mathbf{V}^{-1} \left( \mathbf{H}_{r}^T \mathbf{H}_{r} + \frac{\lambda_a}{2}\mathbf{I} \right)^{-1} (\mathbf{V}^H)^{-1} \mathbf{V}^H \mathbf{H}_{r}^T \mathbf{x}.
\end{equation}
Simplifying using $(\mathbf{V}^H)^{-1}\mathbf{V}^H = \mathbf{I}$ and $\mathbf{V}^{-1} = \frac{1}{2}\mathbf{V}^H$:
\begin{equation}
\mathbf{\beta}_a 
= \frac{1}{2}\mathbf{V}^H 
\left(\mathbf{H}_{r}^T \mathbf{H}_{r} + \frac{\lambda_a}{2} \mathbf{I}\right)^{-1} 
\mathbf{H}_{r}^T \mathbf{x}.
\label{eq:beta_relation}
\end{equation}

The ECA-BLS solution with regularization parameter $\lambda_{r}$ is defined as:
\begin{equation}
\mathbf{\beta}_{r} 
= \left(\mathbf{H}_{r}^T \mathbf{H}_{r} + \lambda_{r} \mathbf{I}\right)^{-1} 
\mathbf{H}_{r}^T \mathbf{x}.
\label{b_r_inverting}
\end{equation}
Comparing this with \eqref{eq:beta_relation}, if we set $\lambda_{r} = \frac{\lambda_a}{2}$ (or equivalently $\lambda_a = 2\lambda_{r}$), we establish the relationship:
\begin{equation}
\mathbf{\beta}_a = \frac{1}{2}\mathbf{V}^H \mathbf{\beta}_{r}.
\end{equation}

\paragraph{Equivalence of outputs.}
Finally, the CA-BLS output is calculated as:
\begin{equation}
\mathbf{H}_a \mathbf{\beta}_a 
= (\mathbf{H}_{r}\mathbf{V}) 
\left(\frac{1}{2}\mathbf{V}^H \mathbf{\beta}_{r}\right).
\end{equation}
Using $\mathbf{V}\mathbf{V}^H = 2\mathbf{I}$, we obtain:
\begin{equation}
\mathbf{H}_a \mathbf{\beta}_a 
= \mathbf{H}_{r} \left( \frac{1}{2} (2\mathbf{I}) \right) \mathbf{\beta}_{r} = \mathbf{H}_{r}\mathbf{\beta}_{r}.
\end{equation}
This completes the proof.
\end{proof}


\section*{S.III Algorithm of proposed ECA-BLS}
The algorithm of our proposed model is: 

\begin{algorithm}[h]
\caption{Algorithm of the proposed ECA-BLS}
\label{alg:ECABLS}
\begin{algorithmic}[1]
\State \textbf{Input:} 
Input matrix $\mathbf{Z} \in \mathbb{R}^{D \times B}$, 
target matrix $\mathbf{x} \in \mathbb{R}^{D \times C}$, 
where the dataset consists of $D$ samples with $B$ features and $C$ classes.

\State \textbf{Parameters:} 
$\lambda_r$,
$a$ ,
$b$,
$c$ ,
$d$.

\State \textbf{Transform:} 
Map real-valued input to the complex domain using phase encoding
${}^{cx}\mathbf{Z} = \exp(j\pi \mathbf{Z})$.

\State \textbf{Find:} 
Feature node outputs ${}^{cx}\mathbf{A}_i \in \mathbb{C}^{D \times b}$,
$i = 1,2,\ldots,a$, using Eq. 2 of the main manuscript.

\State \textbf{Find:} 
Concatenated feature matrix
${}^{cx}\mathbf{A}^a = [{}^{cx}\mathbf{A}_1,{}^{cx}\mathbf{A}_2,\ldots,{}^{cx}\mathbf{A}_a] \in \mathbb{C}^{D \times ab}$.

\State \textbf{Find:} 
Enhancement node outputs ${}^{cx}\mathbf{Y}_k \in \mathbb{C}^{D \times d}$ using Eq. 5 of the main manuscript,
$k = 1,2,\ldots,c$, using random orthogonal complex weights.

\State \textbf{Find:} 
Concatenated enhancement matrix
${}^{cx}\mathbf{Y}^c = [{}^{cx}\mathbf{Y}_1,{}^{cx}\mathbf{Y}_2,\ldots,{}^{cx}\mathbf{Y}_c] \in \mathbb{C}^{D \times cd}$.

\State \textbf{Find:} 
Hidden-layer output matrix
$\mathbf{H} = [{}^{cx}\mathbf{A}^A, {}^{cx}\mathbf{Y}^c] \in \mathbb{C}^{D \times (ab+cd)}$.

\State \textbf{Construct:} 
Real-augmented hidden matrix
$\mathbf{H}_r = [\Re(\mathbf{H}), \Im(\mathbf{H})] \in \mathbb{R}^{D \times 2(ab+cd)}$.

\State \textbf{Output:} 
Compute output weight matrix
\[
\boldsymbol{\beta}_r =
(\mathbf{H}_r^{T}\mathbf{H}_r + \lambda_r \mathbf{I})^{-1}
\mathbf{H}_r^{T}\mathbf{X}.
\]
\end{algorithmic}
\end{algorithm}


\section*{S. IV Complexity analysis of proposed models}
For ECA-BLS, the augmented hidden matrix
$\mathbf{H}_{r} \in \mathbb{R}^{D \times 2(ab+cd)}$ is real-valued.
The computation of $\mathbf{H}_{r}^{T}\mathbf{H}_{r}$ requires
$D \times 2(ab+cd) \times 2(ab+cd)$ real multiplications, leading to a
quadratic term of order $\mathcal{O}\!\left(8D(ab+cd)^2\right)$.
Subsequently, inverting Eq. 12 of the main manuscript results in a $2 (ab+cd)\times2(ab+cd)$ matrix, which incurs a cubic cost of $\mathcal{O}\!\left(8(ab+cd)^3\right)$.
Finally, the multiplication $\mathbf{H}_{r}^{T}\mathbf{x}$ contributes
$\mathcal{O}\!\left(2D(ab+cd)C\right)$ additional real multiplications.
Together, these operations yield an overall training complexity of
$\mathcal{O}\!\left(8(ab+cd)^3 + 8D(ab+cd)^2 + 2D(ab+cd)C\right)$, with the
same asymptotic order for addition operations.

In contrast, CA-BLS operates on a complex augmented hidden matrix
$\mathbf{H}_{a} \in \mathbb{C}^{D \times 2(ab+cd)}$ and employs Hermitian
products and complex-valued matrix inversion. The computation of
$\mathbf{H}_{a}^{H}\mathbf{H}_{a}$ involves complex multiplications,
each equivalent to four real multiplications and two real additions.
Consequently, the quadratic term increases to
$\mathcal{O}\!\left(32D(ab+cd)^2\right)$ real multiplications and
$\mathcal{O}\!\left(24D(ab+cd)^2\right)$ real additions. Similarly,
inverting the $2(ab+cd)\times2(ab+cd)$ complex matrix in Eq. 10 of the main manuscript incurs a cubic cost
of $\mathcal{O}\!\left(32(ab+cd)^3\right)$ real multiplications and
$\mathcal{O}\!\left(24(ab+cd)^3\right)$ real additions. The computation
of $\mathbf{H}_{a}^{H}\mathbf{x}$ further contributes
$\mathcal{O}\!\left(8D(ab+cd)C\right)$ real multiplications and
$\mathcal{O}\!\left(6D(ab+cd)C\right)$ real additions.

Compared to CA-BLS, ECA-BLS reduces the dominant cubic multiplication complexity from $32n^3$ to $8n^3$, achieving an exact reduction of $75\%$. Likewise, the dominant cubic addition complexity decreases from $24n^3$ to $8n^3$, corresponding to an approximate reduction of $66.7\%$. Since the cubic term dominates the overall training cost, ECA-BLS achieves an effective fourfold reduction in computational complexity while preserving an identical decision function, as guaranteed by the theoretical equivalence analysis.

\section*{S.V. Experimental Setup, Compared Models and Datasets}

\textbf{Experimental Setup:} The experimental configuration includes a PC equipped with an Intel(R) Xeon(R) Gold 6226R CPU running at a speed of 2.90~GHz and
featuring 128~GB of RAM. The system runs on the Windows~11 platform and executes tasks using Python~3.11. The dataset is randomly divided into training and testing subsets in a ratio of 70:30, respectively. To tune the model hyperparameters, five-fold cross-validation is employed, followed by a comprehensive grid search over predefined parameter ranges. The regularization parameter $\lambda_{r}$ is selected from $\{10^{-5}, 10^{-4}, \ldots, 10^{5}\}$, Additionally, $a$ denotes the number of feature groups, $b$ the number of nodes within each
feature group, $d$ the number of enhancement nodes in each enhancement group, and $c$ the
total number of enhancement groups.
The parameter $a$ is selected from $\{5, 10, 15, \ldots, 50\}$, $b$ from $\{1, 3, 5, \ldots, 21\}$,
$d$ from $\{5, 15, 25, \ldots, 105\}$, and $c$ is fixed to $1$. The ECA-BLS network utilizes the holomorphic inverse hyperbolic sine activation function, known as $\operatorname{arcsinh}(\cdot)$.

\renewcommand{\thetable}{S.I}
\begin{table*}[htbp]
    \caption{Pairwise win-tie-loss test of proposed models and baseline models on MCD category UCI datasets.}
    \label{win tie loss sign test for linear}
    \resizebox{1.0\textwidth}{!}{
\begin{tabular}{lccccccccccccc}
\hline
\multicolumn{1}{c}{$\downarrow$  \textbf{Model} $\rightarrow$} &
\multicolumn{1}{c}{\textbf{BLS} \textbf{}\cite{chen2017broad}  } &
\multicolumn{1}{c}{\textbf{H-ELM} \cite{tang2015extreme} } &
\multicolumn{1}{c}{\textbf{GEIB} \cite{10906533} } &
\multicolumn{1}{c}{\textbf{F-BLS} \cite{sajid2024intuitionistic}}&
\multicolumn{1}{c}{\textbf{IF-BLS} \cite{sajid2024intuitionistic}}&
\multicolumn{1}{c}{\textbf{KRP-BLS} \cite{10902561}}
\\

\hline

  \textbf{H-ELM} \cite{tang2015extreme} & [$11,1,14$] 	 &  &  \\

\textbf{GEIB} \cite{10906533} & [$13,0,13$] & [$15,1,10$]	&  &  &  &  &  \\

\textbf{F-BLS} \cite{sajid2024intuitionistic} & [$11,0,15$] & [$13,1,12$] & [$11,0,15$]	 &  &  &  \\

\textbf{IF-BLS} \cite{sajid2024intuitionistic} & [$11,3,12$] & [$16,0,10$] & [$14,1,11$] & [$15,1,10$]   \\

 \textbf{KRP-BLS} \cite{10902561} & [$10,2,14$] & \textbf{[$18,0,8$]} & [$12,1,13$] & \textbf{[$18,0,8$]} & [$10,2,14$]

\\

{\textbf{ECA-BLS}$^{\dagger}$} & [$17,1,8$] & \textbf{[$20,1,5$]} & [$17,2,7$] & \textbf{[$20,0,6$]} & [$14,2,10$] & [$14,2,10$]  	 \\

\hline

 \multicolumn{7}{l}{$^{\dagger}$ represents the proposed models.}
\end{tabular}}
\end{table*}

\renewcommand{\thefigure}{S.2}
\begin{figure*}[htbp]
\begin{minipage}{.240\linewidth}
\centering
\subfloat[\label{2d1}]{\includegraphics[scale=0.25]{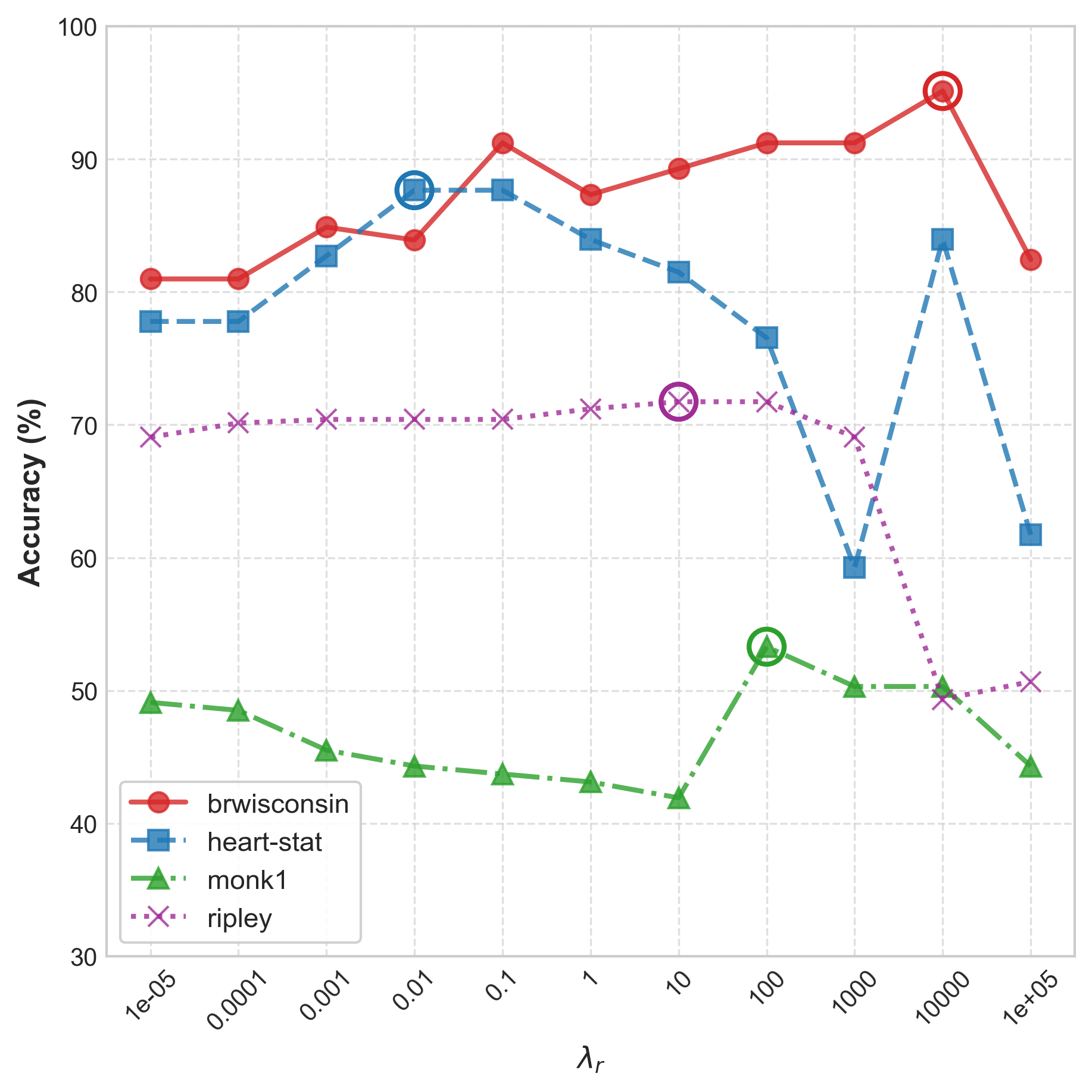}}
\end{minipage}
\begin{minipage}{.240\linewidth}
\centering
\subfloat[\label{2d2}]{\includegraphics[scale=0.25]{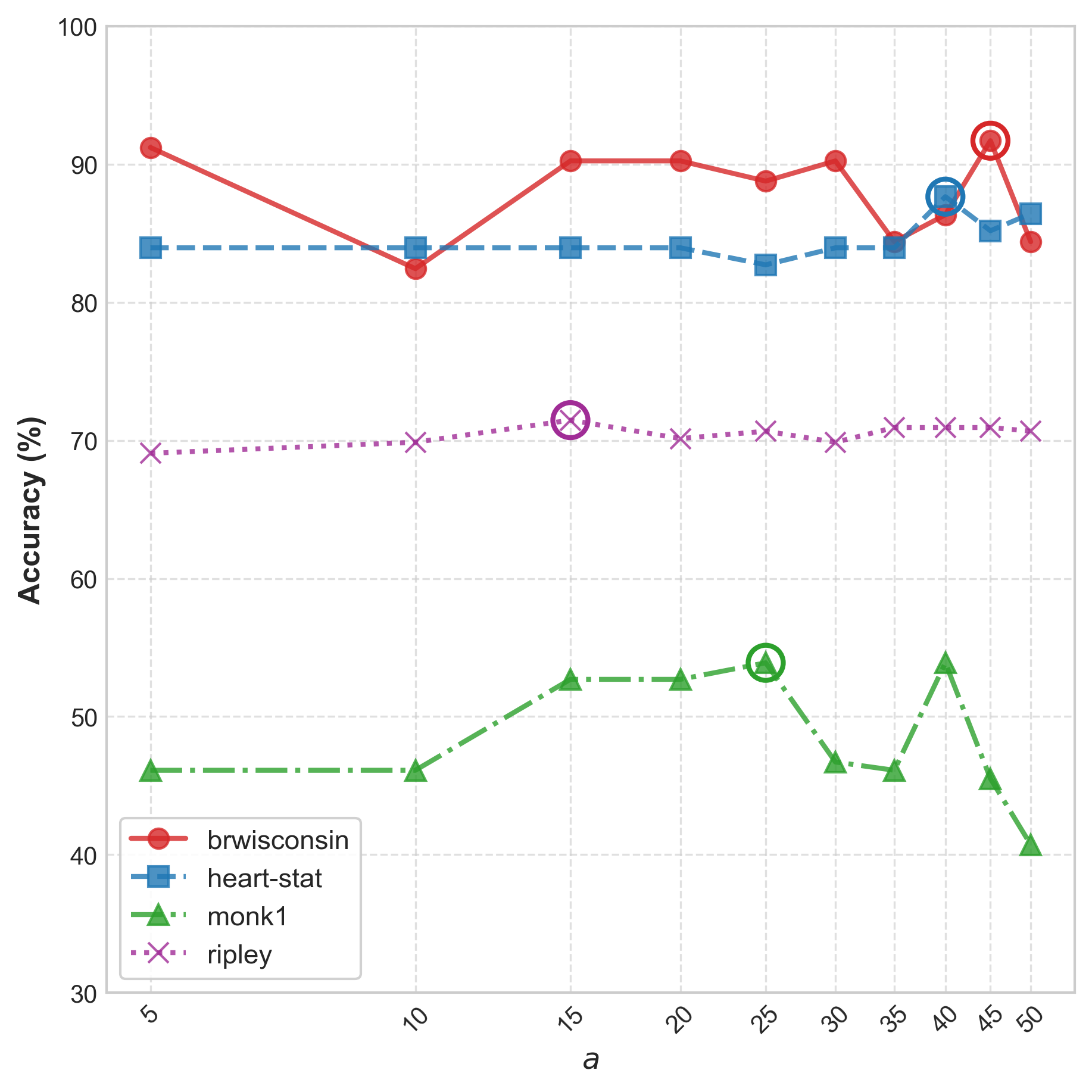}}
\end{minipage}
\begin{minipage}{.240\linewidth}
\centering
\subfloat[\label{2d3}]{\includegraphics[scale=0.25]{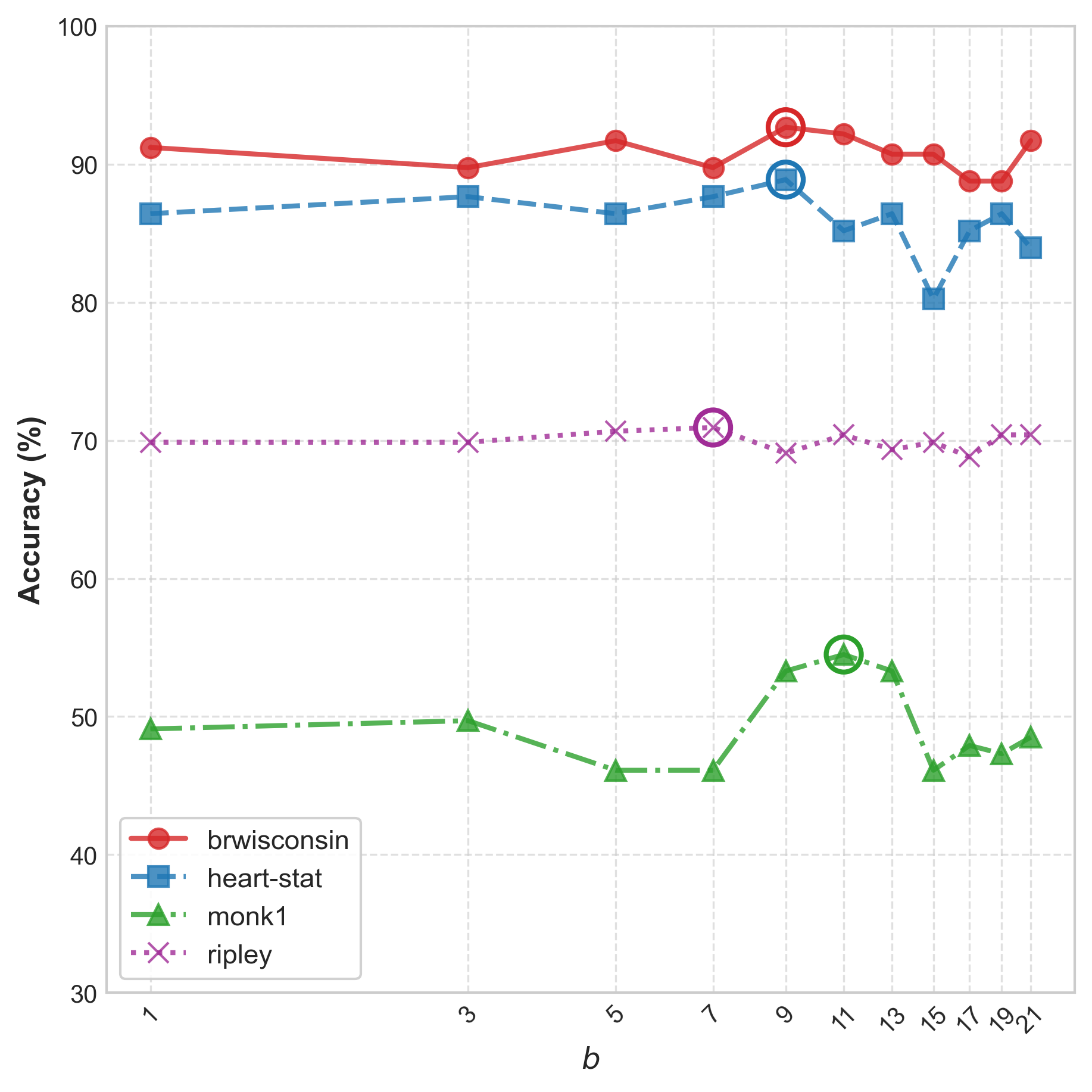}}
\end{minipage}
\begin{minipage}{.240\linewidth}
\centering
\subfloat[\label{2d4}]{\includegraphics[scale=0.25]{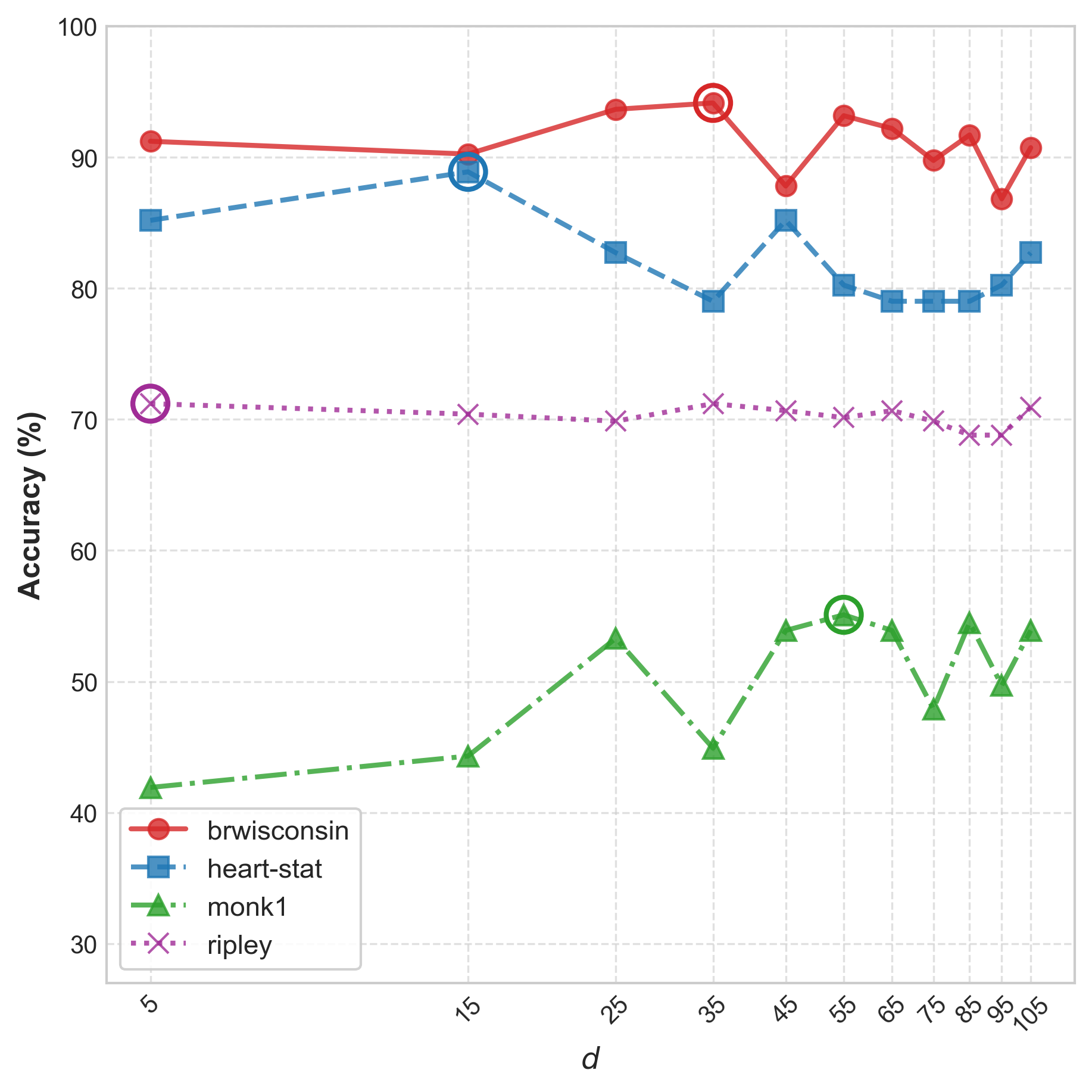}}
\end{minipage}
\par\medskip
\caption{Performance variation of the ECA-BLS model with respect to parameters $\lambda_r$, $a$, $b$ and $d$, respectively.}
\label{2_d_sensi}
\end{figure*}

\renewcommand{\thefigure}{S.3}
\begin{figure*}[htbp]
\begin{minipage}{.240\linewidth}
\centering
\subfloat[brwisconsin\label{3d1}]{\includegraphics[scale=0.35]{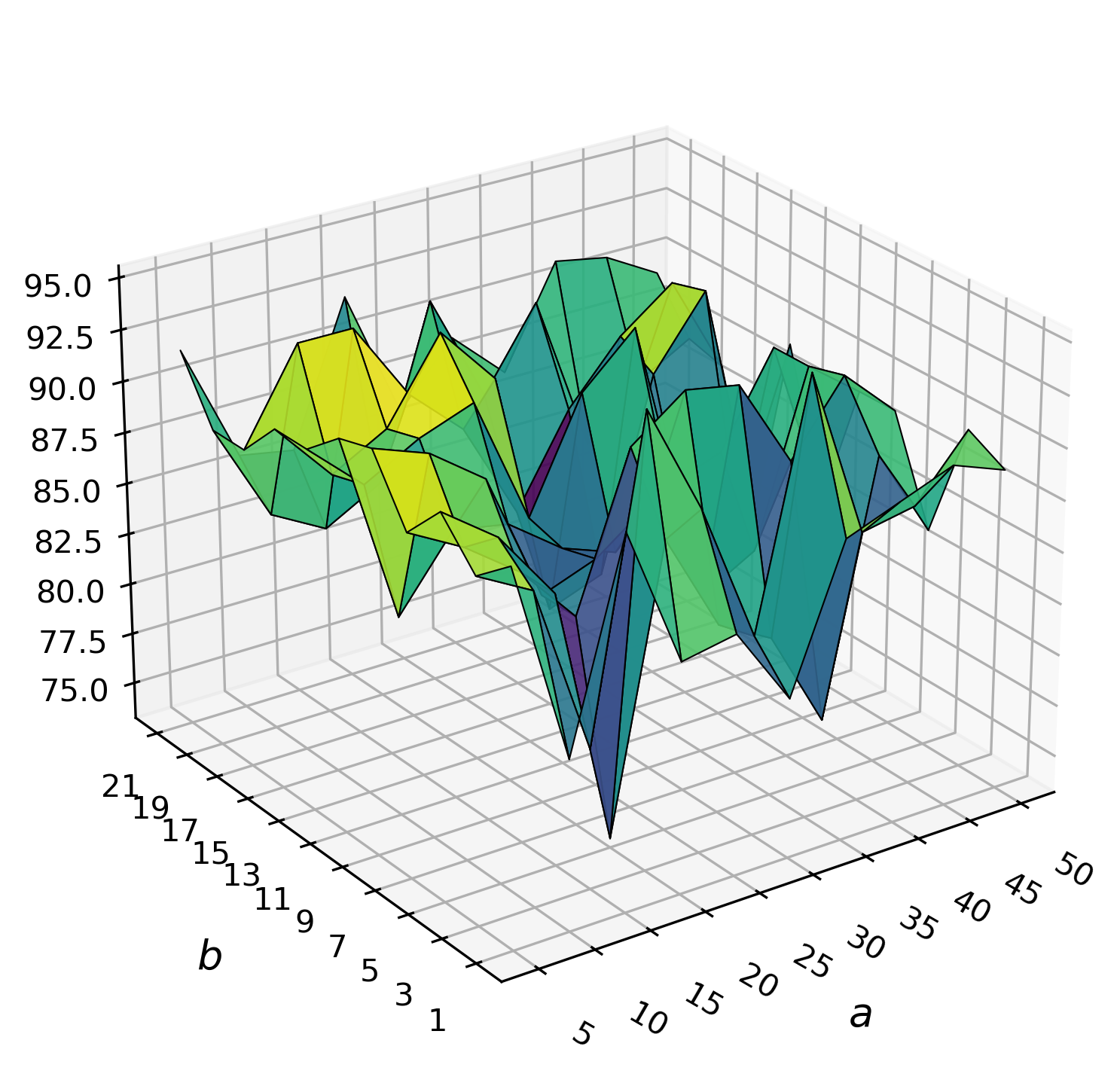}}
\end{minipage}
\begin{minipage}{.240\linewidth}
\centering
\subfloat[heart-stat\label{3d2}]{\includegraphics[scale=0.35]{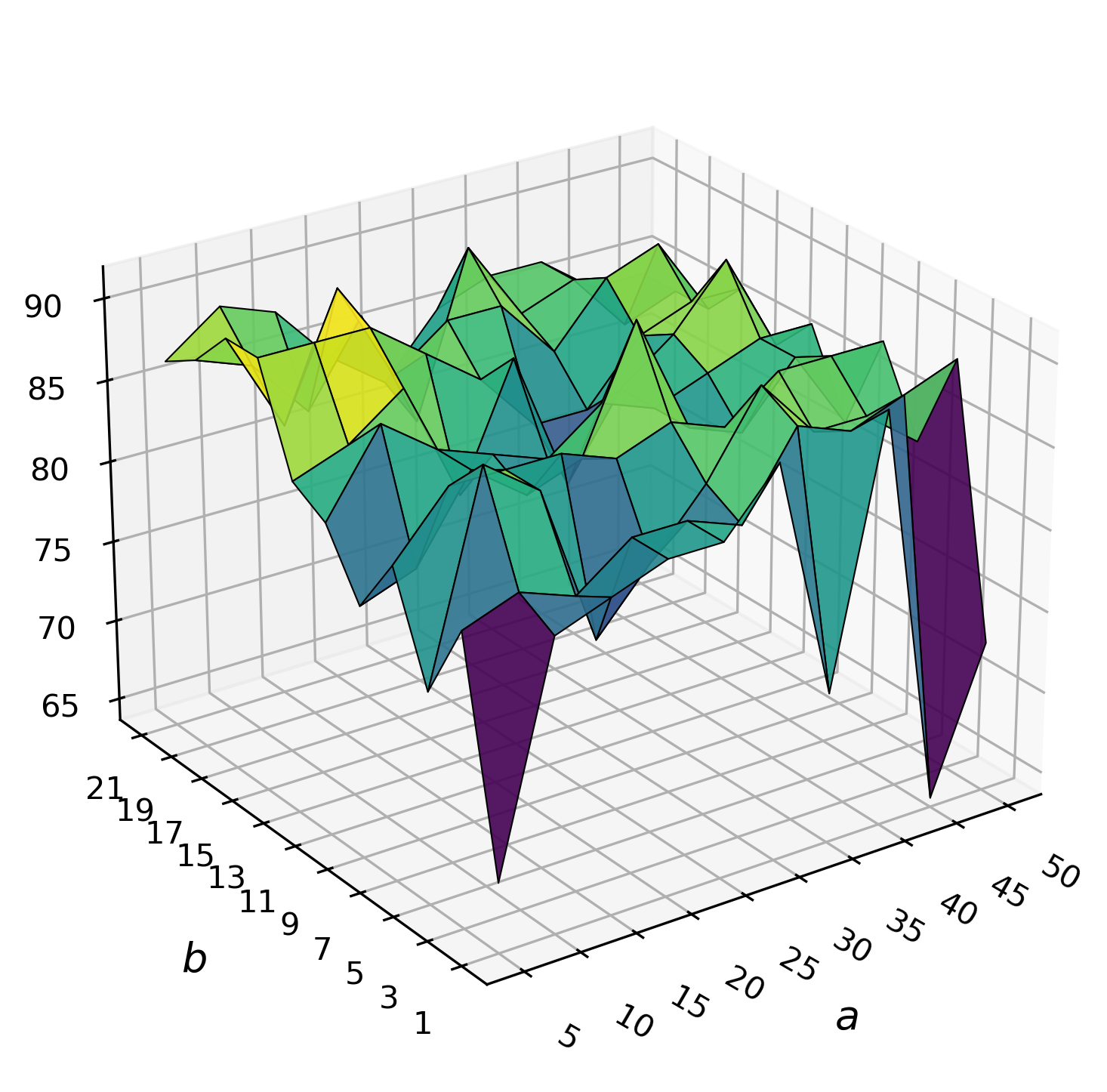}}
\end{minipage}
\begin{minipage}{.240\linewidth}
\centering
\subfloat[monk1\label{3d3}]{\includegraphics[scale=0.35]{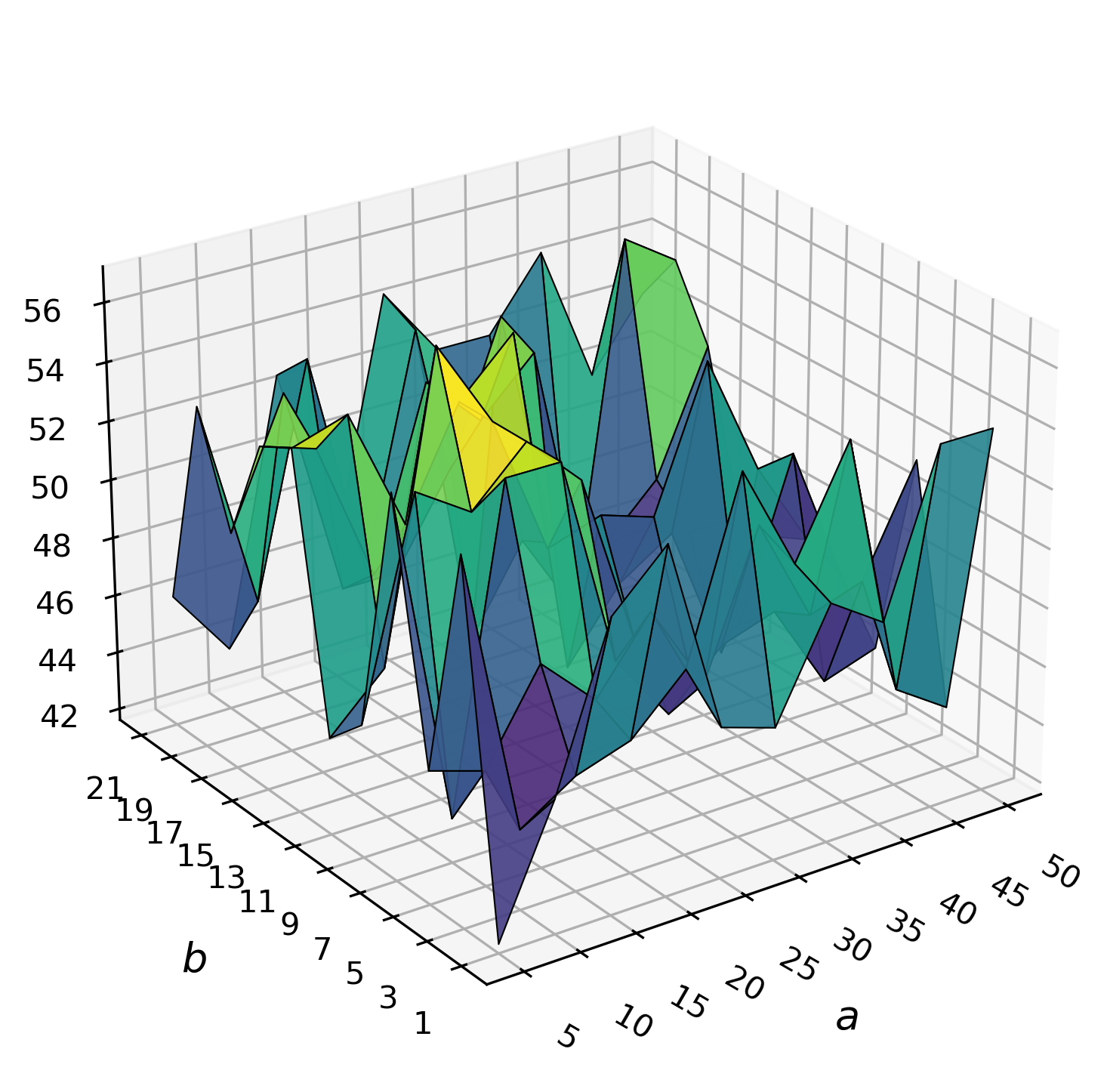}}
\end{minipage}
\begin{minipage}{.240\linewidth}
\centering
\subfloat[ripley\label{3d4}]{\includegraphics[scale=0.35]{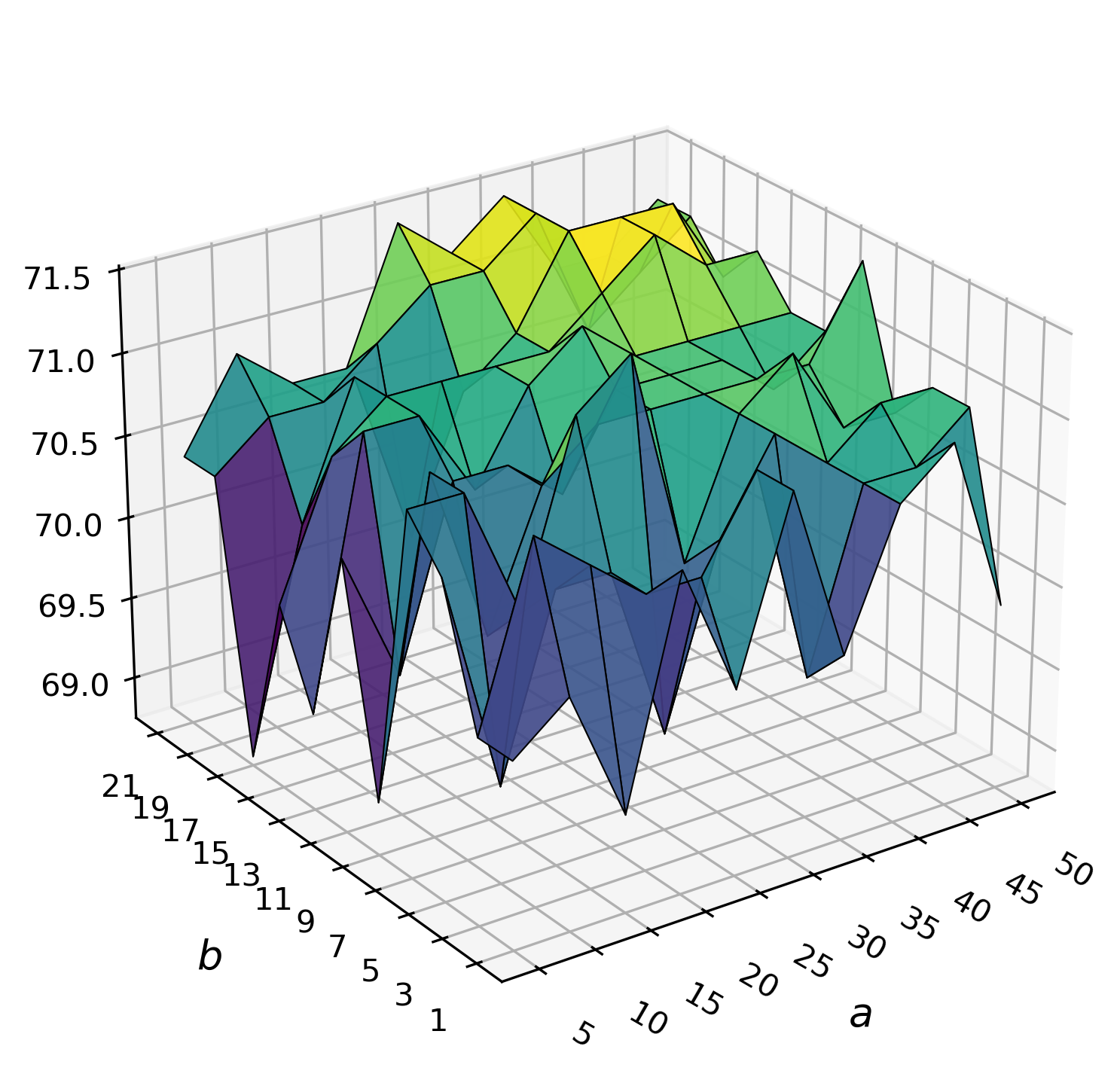}}
\end{minipage}
\par\medskip
\caption{Performance variation of the ECA-BLS model with respect to parameters $a$ and $b$ simultaneously.}
\label{3_d_sensi_a_vs_b}
\end{figure*}

\textbf{Compared Models:} To rigorously assess the effectiveness of the proposed models, we perform a comprehensive experimental comparison against a diverse set of representative random deep and broad learning (RdNN)-based architectures. The selected benchmark methods are systematically categorized to reflect different design philosophies and robustness mechanisms, including the classical Broad Learning System (BLS) \cite{chen2017broad}, Hierarchical Extreme Learning Machine (H-ELM) \cite{tang2015extreme}, fuzzy Broad Learning System (F-BLS) \cite{sajid2024intuitionistic}, intuitionistic fuzzy Broad Learning System (IF-BLS) \cite{sajid2024intuitionistic}, graph-embedding intuitionistic BLS (GEIB) \cite{10906533}, and kernel risk-sensitive mean $p$-power BLS (KRP-BLS) \cite{10902561}.

\textbf{Datasets:} To evaluate the performance of the proposed ECA-BLS model, extensive experiments are conducted on 30 publicly available binary-class datasets obtained from the UCI Machine Learning Repository \cite{dua2017uci} and the KEEL repository \cite{derrac2015keel}. These datasets originate from diverse application domains and exhibit substantial variations in sample size, feature dimensionality, and data complexity. Such diversity enables a systematic and comprehensive assessment of the robustness and generalization capability of ECA-BLS across different data distributions and problem scales.


\textbf{Performance Metrics}
To rigorously evaluate the effectiveness of the proposed PW-RVFL-CIL framework, its performance is assessed using several widely adopted classification metrics, namely Accuracy, Sensitivity, Specificity, Precision, F-measure, and G-mean. These metrics collectively provide a comprehensive and balanced evaluation of predictive performance, particularly in the presence of class imbalance and label noise.

The mathematical definitions of the employed performance measures are given as follows:
\begin{equation}
\text{Accuracy} =
\frac{\text{True}_{+} + \text{True}_{-}}
{\text{True}_{+} + \text{False}_{+} + \text{True}_{-} + \text{False}_{-}},
\end{equation}

\begin{equation}
\text{Sensitivity} =
\frac{\text{True}_{+}}
{\text{True}_{+} + \text{False}_{-}},
\end{equation}

\begin{equation}
\text{Specificity} =
\frac{\text{True}_{-}}
{\text{True}_{-} + \text{False}_{+}},
\end{equation}

\begin{equation}
\text{Precision} =
\frac{\text{True}_{+}}
{\text{True}_{+} + \text{False}_{+}},
\end{equation}

\begin{equation}
\text{F-measure} =
\frac{2 \times \text{Precision} \times \text{Sensitivity}}
{\text{Precision} + \text{Sensitivity}},
\end{equation}

\begin{equation}
\text{G-mean} =
\sqrt{\text{Sensitivity} \times \text{Specificity}}.
\end{equation}

These metrics characterize complementary aspects of classification performance. Here, $\text{True}_{+}$ and $\text{True}_{-}$ denote the number of correctly classified positive and negative samples, respectively, while $\text{False}_{+}$ and $\text{False}_{-}$ represent false positive and false negative outcomes. The inclusion of F-measure and G-mean ensures a fair and informative evaluation by jointly considering both class-wise performance and balance between sensitivity and specificity, which is particularly important for imbalanced learning scenarios.

\renewcommand{\thefigure}{S.4}
\begin{figure*}[htbp]
\begin{minipage}{.240\linewidth}
\centering
\subfloat[brwisconsin\label{3d11}]{\includegraphics[scale=0.35]{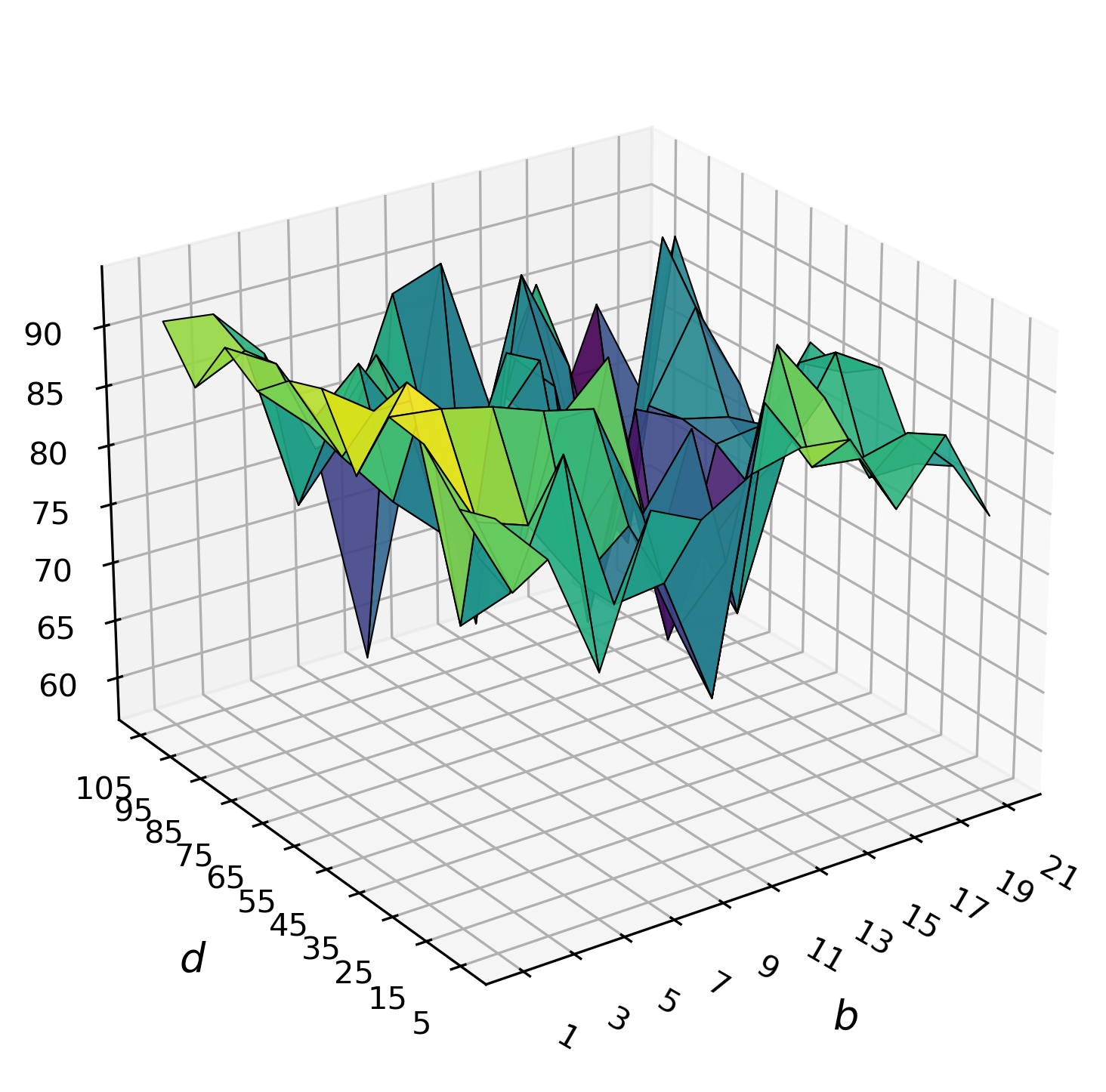}}
\end{minipage}
\begin{minipage}{.240\linewidth}
\centering
\subfloat[heart-stat\label{3d22}]{\includegraphics[scale=0.35]{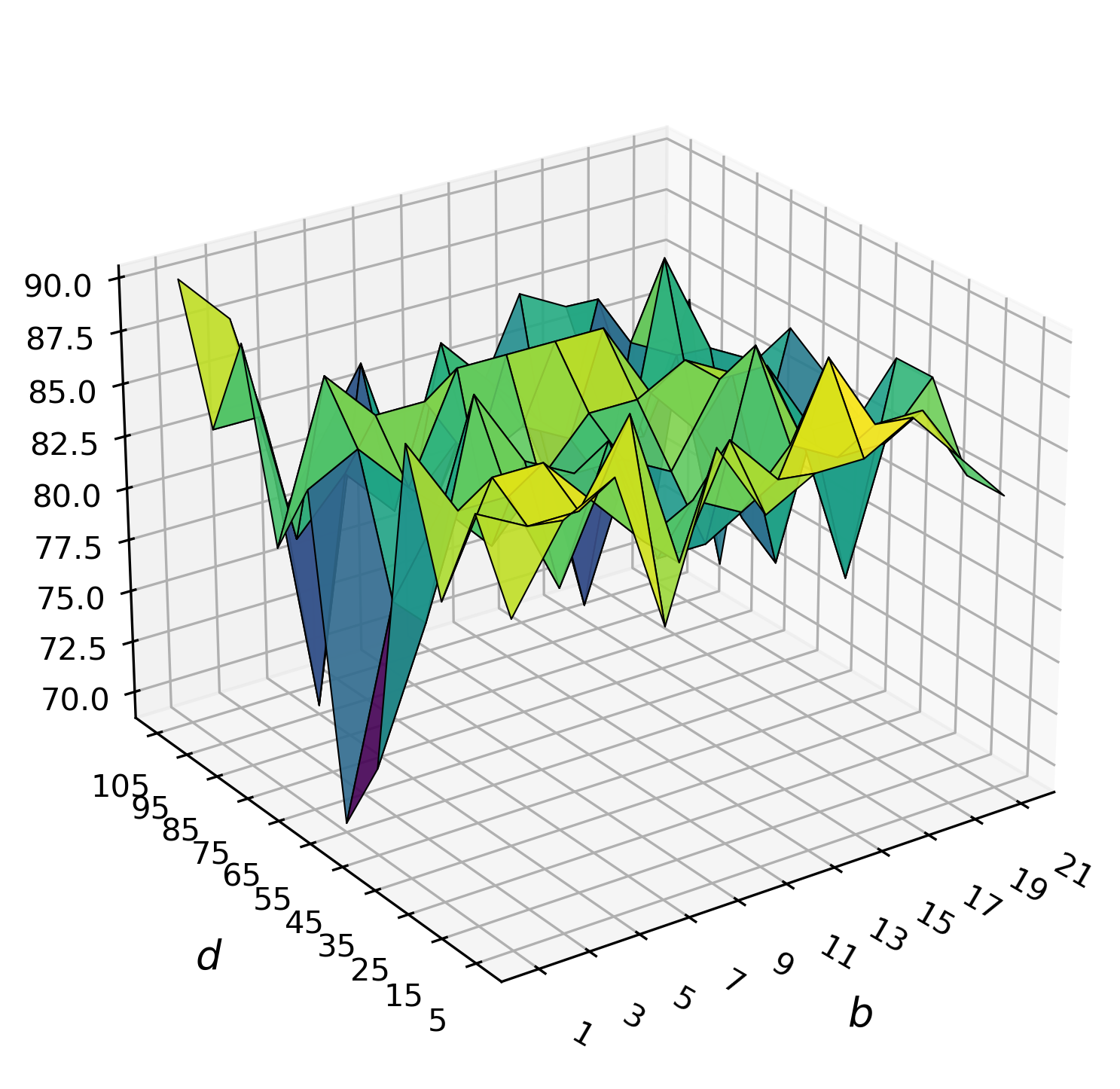}}
\end{minipage}
\begin{minipage}{.240\linewidth}
\centering
\subfloat[monk1\label{3d33}]{\includegraphics[scale=0.35]{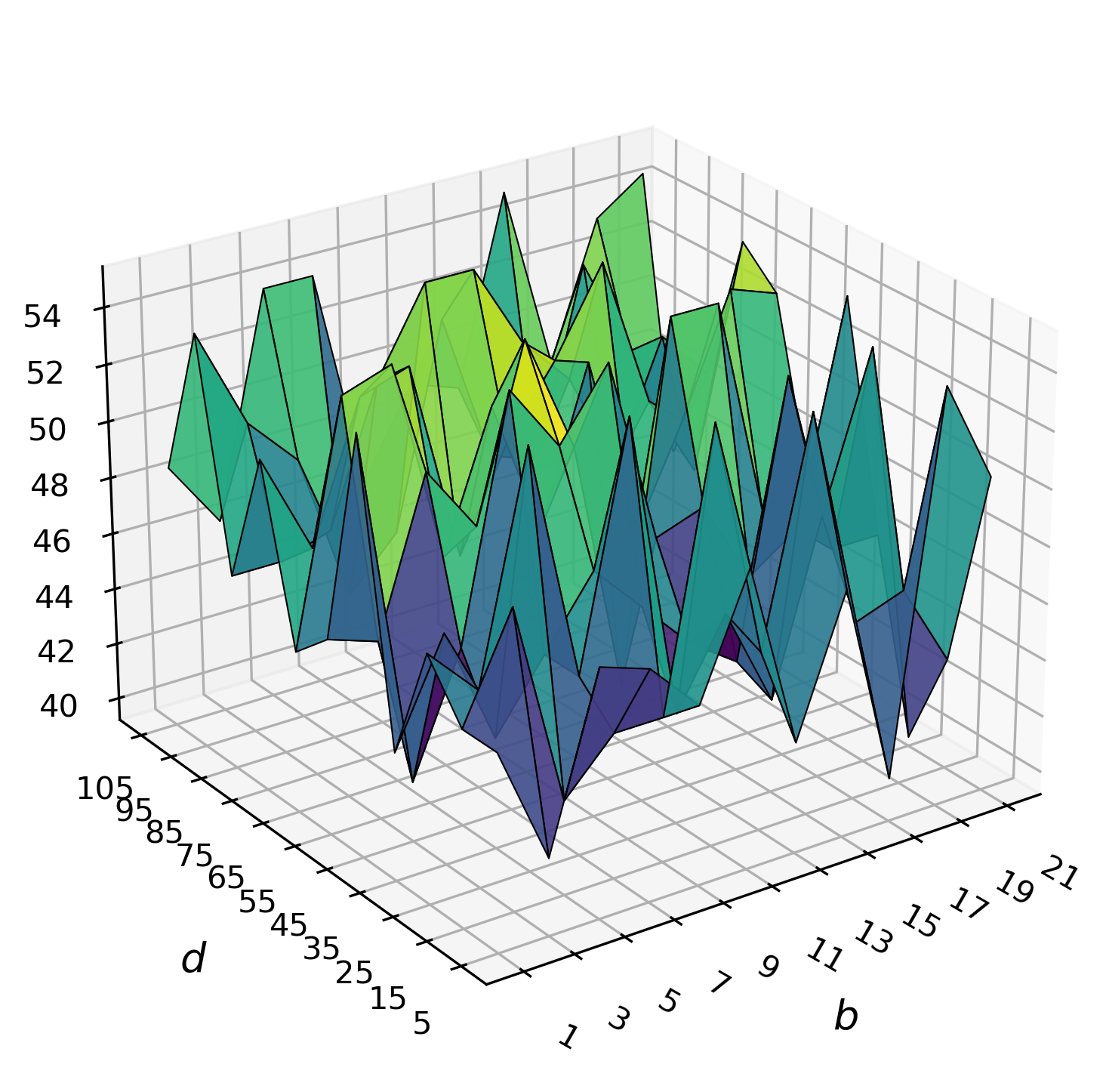}}
\end{minipage}
\begin{minipage}{.240\linewidth}
\centering
\subfloat[ripley\label{3d44}]{\includegraphics[scale=0.35]{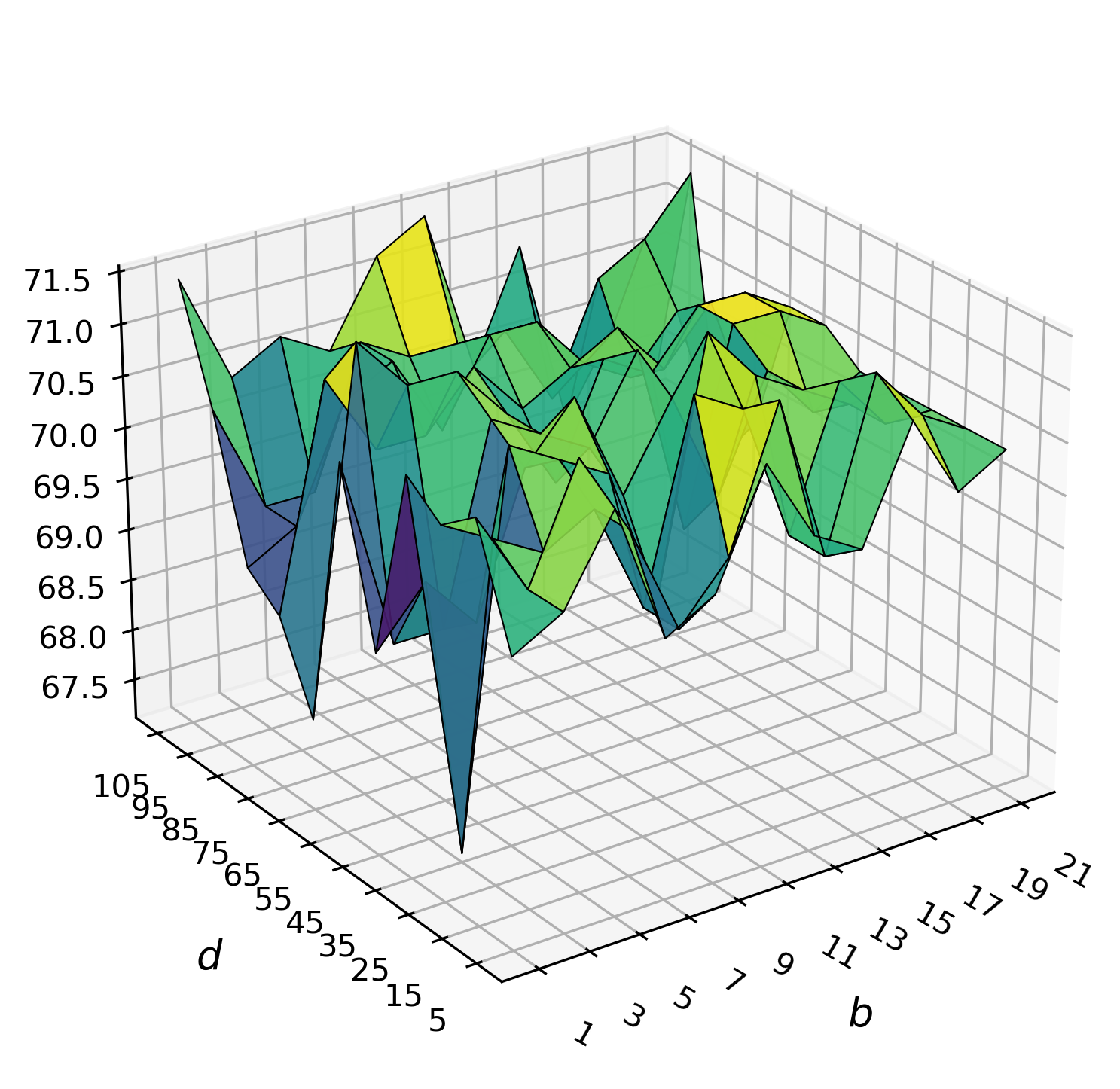}}
\end{minipage}
\par\medskip
\caption{Performance variation of the ECA-BLS model with respect to parameters $b$ and $d$ simultaneously.}
\label{3_d_sensi_b_vs_d}
\end{figure*}

\renewcommand{\thefigure}{S.5}
\begin{figure*}[htbp]
\begin{minipage}{.240\linewidth}
\centering
\subfloat[brwisconsin\label{3d111}]{\includegraphics[scale=0.35]{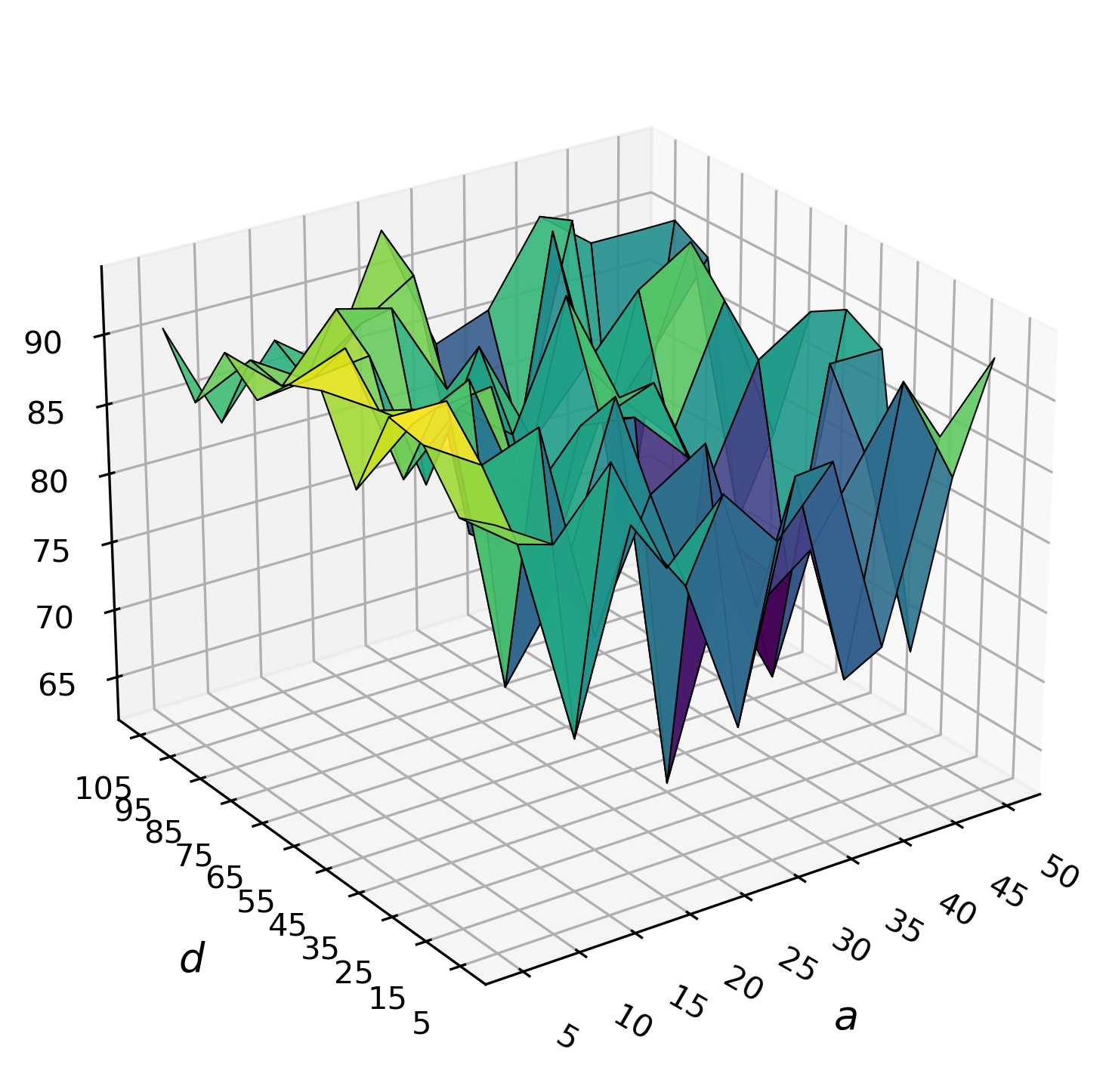}}
\end{minipage}
\begin{minipage}{.240\linewidth}
\centering
\subfloat[heart-stat\label{3d222}]{\includegraphics[scale=0.35]{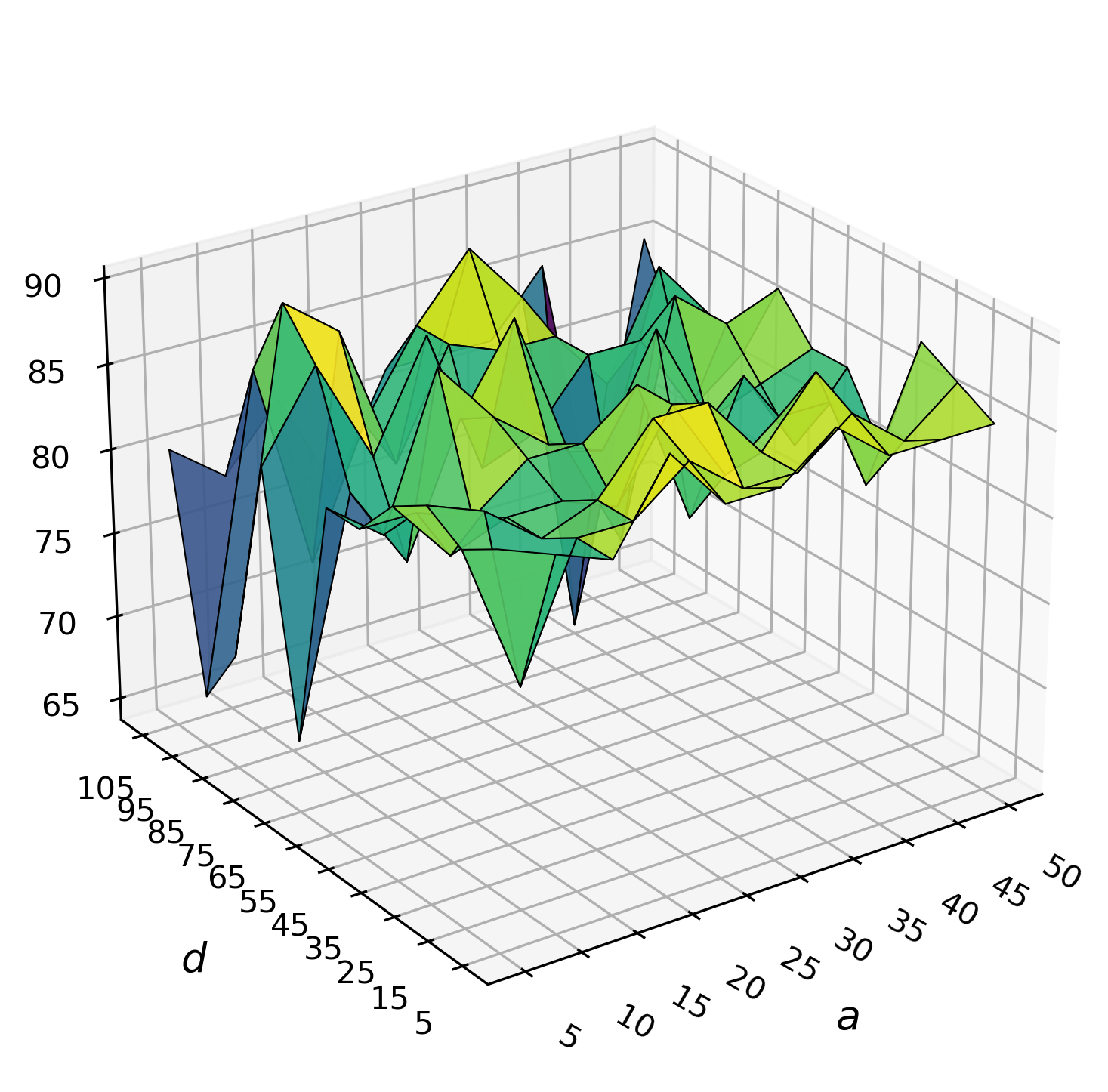}}
\end{minipage}
\begin{minipage}{.240\linewidth}
\centering
\subfloat[monk1\label{3d333}]{\includegraphics[scale=0.35]{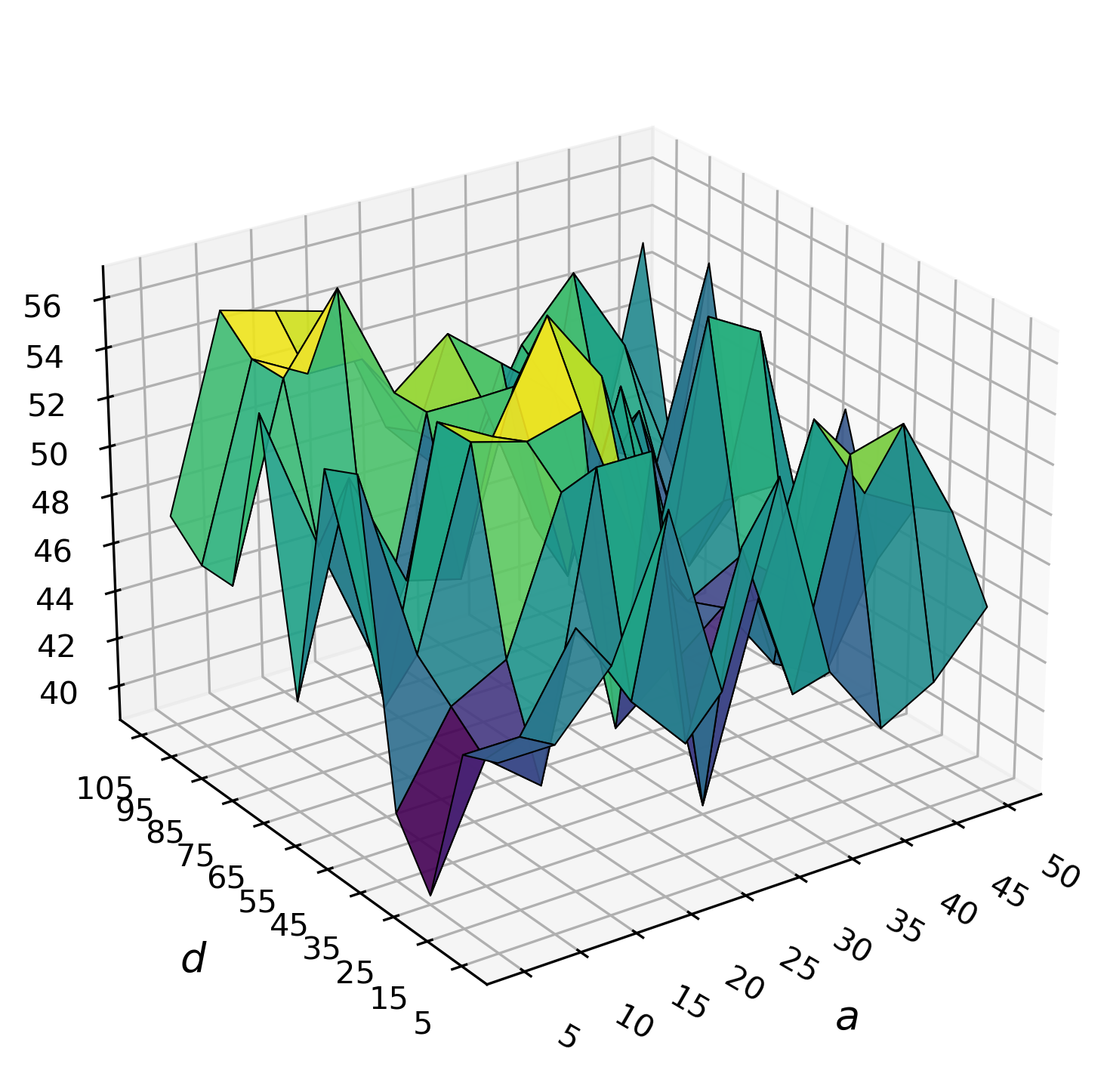}}
\end{minipage}
\begin{minipage}{.240\linewidth}
\centering
\subfloat[ripley\label{3d444}]{\includegraphics[scale=0.35]{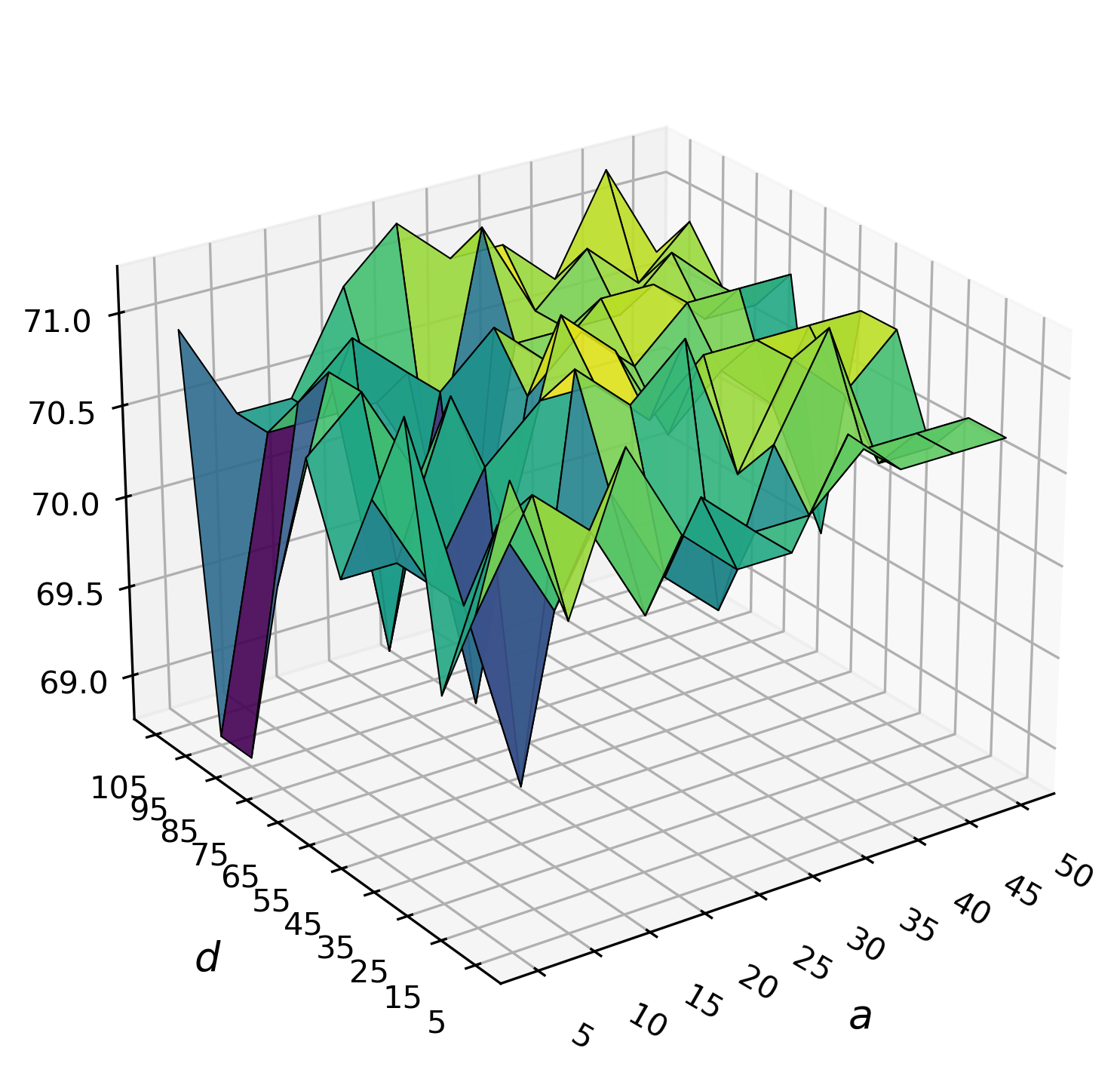}}
\end{minipage}
\par\medskip
\caption{Performance variation of the ECA-BLS model with respect to parameters $a$ and $d$ simultaneously.}
\label{3_d_sensi_a_vs_d}
\end{figure*}


\noindent
\section*{S.VII. Win--Tie--Loss (W--T--L) sign test}  
To further substantiate the statistical significance of the comparative results, we employ the Win--Tie--Loss (W--T--L) sign test \cite{demvsar2006statistical}. Table~\ref{win tie loss sign test for linear} summarizes the pairwise comparisons between the proposed ECA-BLS and the baseline models on the UCI and KEEL datasets. Each entry in the form $[x, y, z]$ indicates that the model listed in the row achieves $x$ wins, $y$ ties, and $z$ losses against the model listed in the corresponding column.

Under the null hypothesis, two competing models are assumed to perform equivalently, implying that each model should win approximately $v/2$ out of $v$ datasets. A statistically significant difference is established if a model achieves at least $v/2 + 1.96\sqrt{v/2}$ wins. In the presence of ties, they are evenly divided between the two models; if the number of ties is odd, one tie is discarded prior to division. For $v=26$ datasets, the resulting significance threshold is 18 wins.

Based on this criterion, the proposed ECA-BLS achieves at least 18 wins against several strong baselines, including F-BLS, GEIB, and H-ELM, thereby demonstrating statistically significant superiority. Although the threshold is not reached against a few baselines, the overall W--T--L outcomes in Table~\ref{win tie loss sign test for linear} consistently favor ECA-BLS. Collectively, these results provide strong evidence that the proposed ECA-BLS model delivers superior and more reliable performance compared to the baseline methods across diverse benchmark datasets.

\section*{S.VIII. Sensitivity Analysis}
\label{sensi}
In this section, we investigate the sensitivity of the proposed ECA-BLS model to its key hyperparameters, namely the regularization parameter $\lambda_r$ and the structural parameters $a$, $b$, and $d$, using four representative datasets: \textit{brwisconsin}, \textit{heart-stat}, \textit{monk1}, and \textit{ripley}.

\subsubsection{Effect of the regularization parameter \texorpdfstring{$\lambda_r$}{lamb}}
Figure~\ref{2d1} illustrates the influence of $\lambda_r$ on classification accuracy. The results indicate that the performance stabilizes when $\lambda_r \geq 0.01$ for all datasets, while excessively large values ($\lambda_r > 1000$) lead to a degradation in accuracy. Accordingly, $\lambda_r$ is selected from the range $[0.01, 1000]$ in subsequent experiments.

\subsubsection{Effect of the number of feature groups \texorpdfstring{$a$}{a}}
Figure~\ref{2d2} shows the effect of the parameter $a$ on model performance. It can be observed that accuracy improves as $a$ increases and reaches its optimum for $a \geq 15$, after which performance gradually declines beyond $a = 45$. Thus, an appropriate range for $a$ is $[15, 45]$.

\subsubsection{Effect of the number of nodes per feature group \texorpdfstring{$b$}{b}}
The impact of parameter $b$ is presented in Fig.~\ref{2d3}. The results reveal that optimal performance is consistently achieved when $b$ lies in the range $[7, 11]$, which is therefore adopted in this study.

\subsubsection{Effect of the number of enhancement nodes \texorpdfstring{$d$}{d}}
Figure~\ref{2d4} depicts the sensitivity of ECA-BLS to the parameter $d$. Unlike the previous parameters, the optimal value of $d$ varies across datasets, indicating that this parameter is inherently dataset dependent.



\renewcommand{\thetable}{S.II}
\begin{table*}[htbp]
\centering
\caption{Best hyperparameters and performance of the proposed model for the experiments on 26 UCI \cite{dua2017uci} and KEEL \cite{derrac2015keel} datasets.}
\label{tab:best_hyperparams}
\renewcommand{\arraystretch}{1}
\resizebox{\textwidth}{!}{%
\begin{tabular}{|l|c|}
\hline
\textbf{{Dataset}} & \textbf{(Sensitivity, Specificity, Precision, Recall, F\_measure, Gmean, lamda\_r, a, b, d)} \\
\hline
acute\_inflammation & (100, 100, 100, 100, 100, 100, 0.1, 30, 21, 45) \\
bank & (0, 100, 0, 0, 0, 0, 100, 10, 1, 25) \\
brwisconsin & (99.21259843, 69.23076923, 84, 99.21259843, 90.97472924, 82.87680319, 10, 5, 1, 5) \\
chess\_krvkp & (90.6374502, 97.37417943, 97.43040685, 90.6374502, 93.91124871, 93.94544874, 0.1, 20, 19, 75) \\
cleve & (80.48780488, 89.79591837, 86.84210526, 80.48780488, 83.5443038, 85.01456555, 1, 10, 19, 5) \\
ecoli-0-1\_vs\_5 & (0, 100, 0, 0, 0, 0, 100, 5, 1, 65) \\
ecoli-0-1-4-7\_vs\_2-3-5-6 & (15.38461538, 98.86363636, 66.66666667, 15.38461538, 25, 38.99973104, 10000, 5, 1, 85) \\
ecoli-0-1-4-7\_vs\_5-6 & (28.57142857, 94.62365591, 28.57142857, 28.57142857, 28.57142857, 51.99550967, 100, 5, 1, 75) \\
ecoli-0-6-7\_vs\_5 & (0, 100, 0, 0, 0, 0, 1000, 5, 1, 95) \\
haber & (43.75, 86.84210526, 41.17647059, 43.75, 42.42424242, 61.63880357, 0.0001, 25, 1, 45) \\
heart\_hungarian & (77.41935484, 82.75862069, 70.58823529, 77.41935484, 73.84615385, 80.04448152, 1, 25, 9, 45) \\
heart-stat & (84.7826087, 80, 84.7826087, 84.7826087, 84.7826087, 82.35659473, 1, 30, 19, 15) \\
ionosphere & (83.75, 80.76923077, 93.05555556, 83.75, 88.15789474, 82.24611284, 0.1, 15, 1, 85) \\
led7digit-0-2-4-5-6-7-8-9\_vs\_1 & (88.88888889, 91.12903226, 42.10526316, 88.88888889, 57.14285714, 90.00199122, 1, 30, 1, 65) \\
mammographic & (78.67647059, 86.92810458, 84.2519685, 78.67647059, 81.36882129, 82.69943448, 1, 20, 7, 15) \\
monk1 & (6.493506494, 95.55555556, 55.55555556, 6.493506494, 11.62790698, 24.90964914, 100000, 30, 11, 45) \\
monk3 & (53.65853659, 29.41176471, 42.30769231, 53.65853659, 47.31182796, 39.72646791, 10000, 35, 11, 35) \\
oocytes\_merluccius\_nucleus\_4d & (88.94472362, 45.37037037, 75, 88.94472362, 81.37931034, 63.52523163, 0.1, 25, 3, 85) \\
ozone & (0, 100, 0, 0, 0, 0, 10000, 20, 1, 95) \\
ripley & (63.68421053, 76.75675676, 73.7804878, 63.68421053, 68.36158192, 69.91561669, 0.00001, 5, 5, 25) \\
shuttle-6\_vs\_2-3 & (66.66666667, 98.48484848, 66.66666667, 66.66666667, 66.66666667, 81.02873913, 10, 50, 15, 65) \\
spambase & (66.41074856, 98.02325581, 95.31680441, 66.41074856, 78.28054299, 80.68331795, 1, 35, 21, 55) \\
spectf & (92.30769231, 43.75, 86.95652174, 92.30769231, 89.55223881, 63.54889093, 100, 5, 3, 105) \\
vertebral\_column\_2clases & (10.71428571, 95.38461538, 50, 10.71428571, 17.64705882, 31.96839098, 10, 45, 7, 55) \\
wpbc & (7.692307692, 97.82608696, 50, 7.692307692, 13.33333333, 27.43188585, 10, 5, 1, 25) \\
yeast-2\_vs\_4 & (0, 100, 0, 0, 0, 0, 1, 5, 1, 25) \\  \hline
\end{tabular}
}
\end{table*}

\subsubsection{Joint effect of \texorpdfstring{$a$ and $b$}{a and b}}
Figure~\ref{3_d_sensi_a_vs_b} illustrates the combined influence of $a$ and $b$ on model performance. The results suggest that medium-to-high values of both parameters lead to improved accuracy. Consequently, $a$ and $b$ are selected from the ranges $[13, 21]$ and $[25, 50]$, respectively.

\subsubsection{Joint effect of \texorpdfstring{$b$ and $d$}{b and d}}
The joint sensitivity of parameters $b$ and $d$ is shown in Fig.~\ref{3_d_sensi_b_vs_d}. The optimal configuration varies across datasets, further confirming the dataset-dependent nature of these parameters.

\subsubsection{Joint effect of \texorpdfstring{$a$ and $d$}{a and d}}
Figure~\ref{3_d_sensi_a_vs_d} presents the combined effect of $a$ and $d$ on ECA-BLS performance. The observed optimal regions differ substantially among datasets, indicating that the interaction between these parameters is also dataset dependent.

Since performance varies across datasets and domains, systematic hyperparameter tuning is necessary to achieve reliable generalization.

\section*{S.IX. Evaluation of the Proposed ECA-BLS under Contaminated Gaussian Noise}

To rigorously examine the robustness of the proposed ECA-BLS framework, we conducted experiments on five representative datasets by injecting contaminated Gaussian noise into the training samples. This evaluation aims to simulate realistic noisy learning environments where data corruption and uncertainty are unavoidable.

As reported in Table~\ref{tab:333}, the proposed ECA-BLS consistently demonstrates superior performance compared to the conventional BLS across all considered noise levels and datasets. Notably, ECA-BLS achieves an overall average accuracy of 76.9522\%, significantly outperforming BLS, which attains 73.7706\%. This clear margin highlights the enhanced noise tolerance of the proposed model.

Furthermore, for each individual dataset and at every noise ratio (ranging from 5\% to 40\%), ECA-BLS either maintains stable performance or degrades gracefully, whereas BLS exhibits noticeable accuracy deterioration as noise intensity increases. This consistent advantage verifies that the adaptive error-aware correction mechanism embedded in ECA-BLS effectively mitigates the adverse effects of contaminated Gaussian noise.

These results strongly demonstrate that the proposed ECA-BLS model possesses superior robustness, stability, and generalization capability under noisy conditions. Such characteristics make ECA-BLS particularly well-suited for real-world applications where training data are often imperfect, noisy, or corrupted, thereby establishing its practical superiority over traditional BLS.

\renewcommand{\thetable}{S.III}
\begin{table}[htp]
\caption{Classification accuracy (\%) of ECA-BLS and BLS under different levels of contaminated Gaussian noise.}
\label{tab:333}
\begin{tabular}{cccc}
\hline
Dataset & Noise & ECA-BLS & BLS \cite{chen2017broad} \\
\hline
\multirow{6}{*}{ecoli-0-6-7\_vs\_5} & 5\% & 90.9091 & 90.9091 \\
 & 10\% & 90.9091 & 90.9091 \\
 & 15\% & 90.9091 & 90.9091 \\
 & 20\% & 90.9091 & 90.9091 \\
 & 30\% & 90.9091 & 89.3939 \\
 & 40\% & 90.9091 & 89.3939 \\
 \hline
\multicolumn{1}{l}{Average accuracy} & \multicolumn{1}{l}{} & \textbf{90.9091} & 90.404 \\
\hline
\multirow{6}{*}{ecoli-0-1-4-7\_vs\_5-6} & 5\% & 94 & 96 \\
 & 10\% & 93 & 92 \\
 & 15\% & 93 & 92 \\
 & 20\% & 93 & 85 \\
 & 30\% & 93 & 90 \\
 & 40\% & 93 & 90 \\
 \hline
\multicolumn{1}{l}{Average accuracy} & \multicolumn{1}{l}{} & \textbf{93.1667} & 90.8333 \\
\hline
\multirow{6}{*}{ecoli-0-1\_vs\_5} & 5\% & 88.8889 & 87.5 \\
 & 10\% & 87.5 & 86.1111 \\
 & 15\% & 88.8889 & 86.1111 \\
 & 20\% & 88.8889 & 88.8889 \\
 & 30\% & 88.8889 & 76.3889 \\
 & 40\% & 88.8889 & 95.8333 \\
 \hline
\multicolumn{1}{l}{} & \multicolumn{1}{l}{} & \textbf{88.6574} & 86.8056 \\
\hline
\multirow{6}{*}{monk1} & 5\% & 53.2934 & 46.1078 \\
 & 10\% & 47.3054 & 48.503 \\
 & 15\% & 46.1078 & 52.6946 \\
 & 20\% & 47.9042 & 44.9102 \\
 & 30\% & 50.2994 & 46.1078 \\
 & 40\% & 56.2874 & 46.1078 \\
 \hline
\multicolumn{1}{l}{Average accuracy} & \multicolumn{1}{l}{} & \textbf{50.1996} & 47.4052 \\
\hline
\multirow{6}{*}{vertebral\_column\_2clases} & 5\% & 48.3871 & 63.4409 \\
 & 10\% & 76.3441 & 60.2151 \\
 & 15\% & 34.4086 & 35.4839 \\
 & 20\% & 69.8925 & 60.2151 \\
 & 30\% & 72.043 & 50.5376 \\
 & 40\% & 69.8925 & 50.5376 \\
 \hline
\multicolumn{1}{l}{Average accuracy} & \multicolumn{1}{l}{} & \textbf{61.828} & 53.405 \\
\hline
\multicolumn{1}{l}{Overall average accuracy} & \multicolumn{1}{l}{} & \textbf{76.9522} & 73.7706\\
\hline
\end{tabular}
\end{table}

\end{document}